\documentclass[11pt]{article}

\usepackage{amsmath,amsfonts,amsthm,amssymb,color}

\usepackage{natbib}
\usepackage{fullpage}

\newtheorem{theorem}{Theorem}[section]
\newtheorem{corollary}[theorem]{Corollary}
\newtheorem{lemma}[theorem]{Lemma}
\newtheorem{definition}[theorem]{Definition}
\newtheorem{problem}{Problem}

\usepackage{hyperref}
\usepackage[ruled,vlined]{algorithm2e}
\usepackage{xspace}
\usepackage{ifthen}

\newcommand{\eps}{\ensuremath{\epsilon}\xspace}
\renewcommand{\tilde}{\widetilde}
\renewcommand{\hat}{\widehat}
\renewcommand{\bar}{\overline}

\newcommand{\R}{\mathbb{R}}
\newcommand{\C}{\mathbb{R}}
\newcommand{\OO}{\mathcal{O}}

\DeclareMathOperator{\tr}{tr}
\newcommand{\norm}[1]{\lVert#1{\rVert}}
\newcommand{\normone}[1]{{\norm{#1}}_1}
\newcommand{\normtwo}[1]{{\norm{#1}}_2}
\newcommand{\norminf}[1]{{\norm{#1}}_\infty}
\newcommand{\fnorm}[1]{{\norm{#1}}_F}
\newcommand{\maxnorm}[1]{{\norm{#1}}_{\max}}
\providecommand{\expect}[2]{\ensuremath{\ifthenelse{\equal{#1}{}}{\mathbb{E}}{\mathbb{E}_{#1}}\left[#2\right]}\xspace}

\newcommand{\maxi}[1]{\mbox{maximize} & {#1 } & \\}
\newcommand{\st}{\mbox{subject to} }
\newcommand{\con}[1]{&#1 & \\}

\newenvironment{lp}{\begin{equation}  \begin{array}{lll}}{\end{array}\end{equation}}
\newenvironment{lp*}{\begin{equation*}  \begin{array}{lll}}{\end{array}\end{equation*}}

\newcommand{\inner}[1]{\langle #1\rangle}
\newcommand{\Us}{U^\star}
\newcommand{\Vs}{V^\star}
\newcommand{\Ms}{M^\star}
\newcommand{\Sigs}{\Sigma^\star}
\newcommand{\sigs}{\sigma^\star}

\begin{document}

\title{Non-Convex Matrix Completion Against a Semi-Random Adversary}
\author{Yu Cheng \qquad Rong Ge \\ Duke University}
\date{}

\maketitle


\begin{abstract}
Matrix completion is a well-studied problem with many machine learning applications. In practice, the problem is often solved by non-convex optimization algorithms. However, the current theoretical analysis for non-convex algorithms relies crucially on the assumption that each entry of the matrix is observed with exactly the same probability $p$, which is not realistic in practice.

In this paper, we investigate a more realistic semi-random model, where the probability of observing each entry is {\em at least} $p$. 
Even with this mild semi-random perturbation, 
we can construct counter-examples where existing non-convex algorithms get stuck in bad local optima.

In light of the negative results, we propose a pre-processing step that tries to re-weight the semi-random input, so that it becomes ``similar'' to a random input. We give a nearly-linear time algorithm for this problem, and show that after our pre-processing, all the local minima of the non-convex objective can be used to approximately recover the underlying ground-truth matrix.

\end{abstract}




\section{Introduction}
\label{sec:intro}
Non-convex optimization techniques are now very popular for machine learning applications, especially in learning neural networks \citep{deepsurvey1,deepsurvey2}. These methods are easier to implement and extremely efficient in practice. 
Even though optimizing a non-convex function is hard in general, recent works proved convergence guarantees for problems including matrix completion, dictionary learning and tensor decomposition \citep{jain2013low, sun2015guaranteed,rgDict2,ge2015escaping}.

Unlike convex optimization, non-convex problems may have many bad local optima. Existing techniques rely heavily on model assumptions to get a strong initialization (e.g., \citep{jain2013low, sun2015guaranteed}), or to prove that the objective function has no bad local optima (e.g., \citep{ge2016matrix, GeJZ17}). In this paper, we investigate the robustness of non-convex algorithms against model misspecification. In particular, we focus on the matrix completion problem \--- a well-known learning problem with applications to recommendation systems \citep{koren2009bellkor,rennie2005fast}. Both its convex relaxations \citep{srebro2005rank,recht2011simpler} and non-convex approaches \citep{jain2013low, sun2015guaranteed,ge2016matrix} were studied extensively before (see Section~\ref{sec:related} for more related work).

In the matrix completion problem, there is an unknown low-rank matrix $\Ms$ that can be factored into $\Ms = \Us(\Vs)^\top$, where $\Us \in \R^{n_1\times r}$, $\Vs = \R^{n_2\times r}$, and $r$ is the rank of the matrix. 
After observing a random set of entries, the goal is to recover the entire low-rank matrix $\Ms$. 

Matrix completion arises naturally in the design of recommendation systems. For example, the rows of the matrix may correspond to users, and the columns of the matrix correspond to items. Each row of $\Us$ is then a vector representing the preference of a user, and each row of $\Vs$ is a vector representing the properties of an item. Revealing an entry $\Ms_{i,j}$ can be interpreted as user $i$ providing his rating for item $j$. If the matrix $\Ms$ can be recovered from few observations, the system can simply recommend the item with the highest rating in $\Ms$ (among all the items that the user has not rated before).

The standard assumption made by both convex and non-convex approaches for matrix completion is that the entries are observed {\em uniformly} at random. That is, every entry is observed with the same probability $p$. This is not realistic in recommendations systems, as different groups of users may have different probabilities of rating different items.

In this paper, we investigate whether the popular approaches for matrix completion can handle {\em model misspecification}. In particular, we show that if the observation probability is always {\em at least} $p$, then the popular non-convex approaches provably fail. In light of this, we give an efficient pre-processing algorithm that reweights the observed entries, and prove that non-convex optimization algorithms applied to the reweighted objective will always find the desired solution.

\subsection{Non-Convex Matrix Completion}

%

Traditionally, matrix completion can be solved by a convex relaxation \citep{candes2010power}.
\[
\min \; \|M\|_* \qquad \mathrm{s.t.} \; M_{i,j} = \Ms_{i,j}, \; \; \forall (i,j)\in \Omega.
\]
Here $\|\cdot\|_*$ is the nuclear norm of a matrix and $\Omega$ is the set of observed entries.
However, in practice the matrix $\Ms$ can have very high dimensions. Solving the convex relaxation can be very expensive. The most popular approaches in practice use non-convex heuristics: The algorithms represent the matrix $M$ as $UV^\top$ where $U,V$ have the same dimensions as $\Us, \Vs$, and then try to minimize 
\begin{align*}
\sum_{(i,j)\in \Omega} & \left((UV^\top)_{i,j} - \Ms_{i,j}\right)^2,
\end{align*}
using techniques such as alternating minimization \citep{koren2009bellkor} or gradient descent \citep{rennie2005fast}. Sometimes additional terms are added as regularizers. Recently, many results established strong convergence guarantees for these non-convex approaches (e.g., \citep{jain2013low,sun2015guaranteed,ge2016matrix}, see more discussions in Section~\ref{sec:related}).

In order to understand robustness of algorithms under model misspecification, we consider a natural setting where the probability $p_{i,j}$ of observing entry $(i,j)$ can be different, but they are still all at least $p$. In fact, our algorithm can work with a slightly stronger semi-random model: each entry is first revealed with probability $p$ (same as the standard model). After that, an adversary is allowed to examine the ground-truth matrix $\Ms$ and the set of currently observed entries. The adversary can choose to reveal additional entries of the matrix (adding elements to $\Omega$). 
The setting where every entry is observed with probability $p_{i,j} \ge p$ is a special case of this semi-random model.

Intuitively, what the adversary does is beneficial for us, because we get to observe more entries of the matrix. Indeed, the convex relaxation will still work in this semi-random model.\footnote{To see this, notice that the additional entries correspond to additional constraints, and the original optimal solution ($\Ms$) also satisfies all of these additional constraints. Therefore $\Ms$ must still be the optimal solution.} Our work is motivated by the following questions: Are the non-convex approaches robust in this semi-random model? If not, is there a way to fix the non-convex algorithms to get an algorithm that is both robust to the semi-random adversary and efficient in practice?

\subsection{Our Results}
\label{sec:results}
%
%
%
%
%

We first give some examples where the non-convex approaches fail.
When each entry is revealed with equal probability, \cite{GeJZ17} showed that all local minima of a non-convex objective (see Equation \eqref{eqn:asymmetricobj} in Section~\ref{sec:matprelim}) recovers the ground-truth matrix $\Ms$. We give an example where Equation \eqref{eqn:asymmetricobj} has a local minimum that is very far from $\Ms$ in the semi-random model. Another popular approach in non-convex approaches is to rely on SVD to do initialization. In the second example, we show that in the semi-random model, the SVD of the observed entries can be very far from the ground truth $\Ms$. These examples show that the current non-convex approaches rely heavily on the assumption that entries are revealed independently. We leave the details to Appendix~\ref{app:examples}.

We then give an efficient pre-processing algorithm that can make non-convex algorithms robust to the semi-random adversary. Intuitively, we view the observed entries as edges of a bipartite graph (between rows and columns of the matrix): the initial random procedure generates a random bipartite graph with independent edges, and the adversary can add additional edges. We design an algorithm that will {\em re-weight} the edges, so that the resulting weighted graph ``looks like'' a random bipartite graph in terms of its spectral properties.
In Section~\ref{sec:matcomp}, we will show that such spectral properties are sufficient for non-convex approaches to approximately recover the ground-truth $\Ms$.

\begin{problem}[Removing Noisy Edges]
\label{prob:main}
Let $G = (V_1,V_2, E)$ be an unweighted bipartite graph. Assume there exists a subset $S \subseteq E$ of edges such that $(V_1,V_2, S)$ is spectrally similar to a complete bipartite graph. Compute a set of weights $w_e \ge 0$ for every edge $e\in E$ such that the weighted graph is spectrally similar to a complete bipartite graph.
\end{problem}

The existence of weights $w_e$ in Problem~\ref{prob:main} is guaranteed by assigning weight $1$ to the edges selected randomly, and weight $0$ to the edges added by the adversary. Roughly speaking, the pre-processing algorithm will try to put large weights on the edges generated originally, and assign small weights to the edges added by the adversary. 
Note that these weights cannot be found by simple reweighting schemes based on the number of observed entries in each row/column, because our counter-examples have the same number of observed entries in every row/column. 

We give a nearly-linear time algorithm for Problem~\ref{prob:main} in Section~\ref{sec:bss}.
Our approach extends techniques from a line of work on linear-sized graph sparsification algorithms \citep{BatsonSS12, AllenLO15, LeeS17}.


\begin{theorem}[Pre-processing (Informal)]
\label{thm:preprocess-informal}
Given an input graph with $m$ edges, assume a subset of the edges forms a graph that is $\beta$-spectrally close to the complete bipartite graph. There is an algorithm that runs in time $\tilde{O}(m/\epsilon^{O(1)})$ that outputs a set of weights for the edges, and the weighted graph is $O(\beta)+\epsilon$ spectrally close to the complete bipartite graph.~\footnote{Throughout the paper, $\tilde O(\cdot)$ hides polylogarithmic factors.}
\end{theorem}

See Definition~\ref{def:spectral} for the definition of spectrally similarity, and Theorem~\ref{thm:bss} for the formal version of the theorem. 
After this pre-processing step, we get a weighted graph that is spectrally similar to the original random graph (before the adversary added edges). We can then use these weights to change the non-convex objective, and minimize:
\[
\sum_{(i,j)\in \Omega} W_{i,j} \left((UV^\top)_{i,j} - \Ms_{i,j}\right)^2.
\]

We will show this objective (with some additional regularizers) has no bad local minimum.

\begin{theorem}[Main]\label{thm:main}
In the semi-random model, if the reveal probability $p \ge \frac{C\mu^6r^6(\kappa^\star)^6\log (n_1+n_2)}{n_1\epsilon^2}$ for a large enough constant $C$, with high probability, using weights produced by Algorithm~\ref{alg:ls17}, all local minima of Objective~\eqref{eqn:asymmetricobj} satisfies $\|UV^\top - \Ms\|_F^2 \le \epsilon \|\Ms\|_F^2$. The pre-processing time is $\tilde O(m \cdot \mathrm{poly}(\mu, r, \kappa^\star, \eps^{-1}))$ where $m$ is the total number of revealed entries. 
\end{theorem}

Here $\mu, r, \kappa^\star$ are the incoherence parameter, rank, and the condition number of $\Ms$ (see Section~\ref{sec:prelim} for formal definitions).
Previous analyses crucially rely on concentration bounds obtained from observing the entries of $\Ms$ uniformly at random.
We prove Theorem~\ref{thm:main} in Section~\ref{sec:matcomp}, where we replace these concentration bounds with spectral properties guaranteed by our pre-processing algorithm (Theorem~\ref{thm:preprocess-informal}).
See Section~\ref{sec:matcomp} and Appendix~\ref{app:matrix} for details.

\subsection{Related Work}
\label{sec:related}
\paragraph{Matrix Completion.} The first theoretical guarantees on matrix completion come from convex relaxation~\citep{srebro2005rank,recht2011simpler,candes2010power,candes2009exact, negahban2012restricted}.
In particular, \cite{candes2010power} and \cite{recht2011simpler} showed that if $\Omega((n_1+n_2)r  \mu^2  \log^2 (n_1+n_2))$ entries are observed randomly (where $\mu$ is the incoherence parameter, see Section~\ref{sec:matprelim}), the nuclear norm convex relaxation recovers the exact underlying low-rank matrix. There have been many works trying to improve the running time (e.g., \citep{srebro2004maximum,mazumder2010spectral,hastie2014matrix} and the references therein).

For the non-convex approaches, the first set of results require a good initialization. \cite{keshavan2010matrix,keshavan2010matrixnoisy} showed that well-initialized gradient descent recovers $\Ms$. Later, it was shown that several other non-convex algorithms, including alternating minimization
\citep{jain2013low,hardt2014fast,hardt2014understanding} and gradient descent \citep{chen2015fast}, converge to the correct solution with a careful initialization. 

Recently, the work of \cite{sun2015guaranteed} (and subsequent works \citep{zhao2015nonconvex,zheng2016convergence,tu2015low}) established a common framework for matrix completion algorithms with a good initialization. In particular, they proved an analog of strong convexity in the neighborhood of the optimal solution. As a result, many different algorithms can converge to $\Ms$. 

For guarantees without careful initialization, \cite{DBLP:conf/icml/SaRO15} analyzed stochastic gradient descent from random initialization. More recently, \cite{ge2016matrix,park2016non,GeJZ17,chen2017memory} showed that the non-convex objective (with careful regularization) does not have any bad local minima.

All of the works above require uniformly random observations. There have also been works that try to solve matrix completion problem under deterministic assumptions \citep{BhojanapalliJ14, LiLR16}. \citep{BhojanapalliJ14} uses convex relaxations, and the conditions in \citep{LiLR16} does not apply to the semi-random model.

\paragraph{Graph Sparsification.}
The goal of graph sparsification is to use a few weighted edges to approximate a given graph. We focus on the notion of spectral similarity in this paper (see Definition~\ref{def:spectral}).~\footnote{For an undirected graph $G$, we use $m$ and $n$ to denote the number of edges and vertices respectively.} The seminal work of~\cite{SpielmanT11} showed that, for any undirected graph $G$, an $\eps$-spectral sparsifier of $G$ with $\tilde O(n / \eps^2)$ edges can be constructed in nearly-linear time. In a beautiful piece of work, \cite{BatsonSS12} showed that an $\eps$-spectral sparsifier with a linear number of $O(n/\eps^2)$ edges exist and can be computed in polynomial time.
Since then, there have been many efforts \citep{AllenLO15, LeeS15, LeeS17} on speeding up the construction of linear-sized sparsifiers. Recently, \cite{LeeS17} gave an algorithm for this problem that runs in near-linear time $\tilde{O}(m/\epsilon^{O(1)})$.

\paragraph{Semi-Random Model.} The semi-random model was first proposed by \cite{blum1995coloring} as an intermediate model between average-case and worst-case. Algorithms for semi-random models were developed for many graph problems, including planted clique~\citep{Jerrum92, feige2000finding}, community detection/stochastic block model~\citep{feige2001heuristics,perry2017semidefinite,moitra2016robust}, graph partitioning~\citep{Kucera95,makarychev2012approximation} and correlation clustering~\citep{mathieu2010correlation,makarychev2015correlation}. Most of these works use convex relaxations, and non-convex approaches (including spectral algorithms) were known to fail~\citep{feige2001heuristics,moitra2016robust} for some of these problems. To the best of our knowledge, our work is the first one that tries to fix the non-convex approach using a light-weight convex pre-processing step.

\paragraph{Non-Convex Optimization.} Although non-convex optimization is NP-hard in general, under reasonable assumptions it is possible to find a local minimum efficiently (e.g., \citep{ge2015escaping,agarwal2016finding,jin2017escape}). It follows from Theorem~\ref{thm:main} that, by running these algorithms after our pre-processing step, we get non-convex algorithms that can recover the ground-truth matrix $\Ms$ approximately even in the semi-random model.


\section{Preliminary}
\label{sec:prelim}

\subsection{Notations}
We use $[n]$ to denote the set $\{1, \ldots, n\}$.
We use $e_i$ or $\chi_i$ for the $i$-th standard basis vector.
We write $I$ for the identity matrix, and $J$ for the all ones matrix.

For a vector $x$, we use $\normone{x}$, $\norm{x}$ and $\norminf{x}$ for the $\ell_1$, $\ell_2$, and $\ell_\infty$ norm of $x$ respectively.

For a matrix $A$, we use $A_i$ to denote $i$-th row of $A$.
We use $\normone{A}$, $\norm{A}$, $\norminf{A}$, $\fnorm{A}$, and $\maxnorm{A}$ for the $\ell_1$, spectral, $\ell_\infty$, Frobenius norm, and maximum absolute entry of $A$. Note that $\norminf{A}$ (and $\normone{A}$) is just the maximum $\ell_1$-norm of the rows (and columns) of $A$.
Let $\lambda_{\min}(A)$ denote the minimum eigenvalue of $A$, and let $\sigma_i(A)$ denote the $i$-th largest singular value of $A$.

A symmetric $n \times n$ matrix $A$ is said to be positive semidefinite (PSD) if $x^\top A x \ge 0$ for all $x \in \R^n$,
and positive definite if $x^\top A x > 0$ for any $x \neq 0$.
For two symmetric matrices $A$ and $B$ of the same dimensions, we write $A \preceq B$ (or equivalently $B \succeq A$) when $B - A$ is positive semidefinite, and $A \prec B$ when $B - A$ is positive definite.

We use $\tr(A)$ for the trace of a square matrix $A$.
Let $A*B$ be the Hadamard (entry-wise) product of two matrices (where $(A*B)_{i,j} = A_{i,j}B_{i,j}$).
For $n_1 \times n_2$ matrices $A$ and $B$, we write $A \bullet B$ or $\inner{A, B}$ for the entry-wise inner product of $A$ and $B$: $\inner{A, B} = \tr(A^\top B) = \sum_{i,j} (A*B)_{i,j}$.
For a subset of entries $\Omega \subseteq [n_1] \times [n_2]$, let $\inner{A,B}_\Omega = \sum_{(i,j)\in \Omega} (A*B)_{i,j}$.
For a weight matrix $W$, let $\inner{A,B}_W = \sum_{i,j} W_{i,j} (A*B)_{i,j} = \inner{W * A, B}$.
Let $\|A\|_\Omega^2 = \inner{A,A}_\Omega$ and $\|A\|_W^2 = \inner{A,A}_W$.

Let $A \otimes B$ denote the Kronecker product of $A$ and $B$.
If $X\in \R^{n\times r}$, define the Katri-Rao product $X\odot X$ be an $n\times r^2$ matrix whose $i$-th row is equal to $X_i\otimes X_i$.~\footnote{This is the transpose of the traditional Katri-Rao product that works on columns instead of rows.} 


\subsection{Matrix Completion and Non-convex Formulation}
\label{sec:matprelim}

Throughout the paper, the ground-truth matrix is always denoted by $\Ms$.
We assume the hidden low-rank matrix $\Ms\in \R^{n_1\times n_2}$ is of rank $r$ and can be decomposed as $\Ms=\Us(\Vs)^\top$, where $\Us \in \R^{n_1\times r}, \Vs\in\R^{n_2\times r}$ (or $\Ms = \Us(\Us)^\top$ in symmetric case).
For the symmetric case, we assume $n_1 = n_2 = n$. For the asymmetric case we assume (w.l.o.g.) that $n_1\le n_2$, and let $n = n_1+n_2$.
We also assume w.l.o.g. that $(\Us)^\top \Us=(\Vs)^\top \Vs$.~\footnote{This can be achieved by setting $\Us = XD^{1/2}$ and $\Vs = YD^{1/2}$ if $\Ms =  XDY^\top$ is the SVD of $\Ms$.}
We use $\sigs_1$, $\sigs_r$ to denote the first and $r$-th singular values of $\Ms$ respectively. Let $\kappa^\star = \sigs_1/\sigs_r$ be the condition number of $\Ms$.

\begin{definition}[Incoherence]
The matrix $\Ms$ with SVD $\Ms = X D Y^\top$ is $\mu$-incoherent, if for all $i\in[n_1],j\in[n_2]$, $\|X_i\| \le \sqrt{\mu r/n_1}$ and $\|Y_j\| \le \sqrt{\mu r/n_2}$.
\end{definition}

In the standard matrix completion setting, each entry of the matrix is observed with probability $p$. Let $\Omega$ be the set of observed entries. 
To make the notation more flexible, we define a weight matrix $W$ such that $W_{i,j} = \frac{1}{p}$ if $(i,j)\in \Omega$ (and 0 otherwise).
We use the objective function in \citep{ge2016matrix,GeJZ17} that includes additional regularizers:
\begin{equation}
\min \; f(U, V) = 2\|UV^\top - \Ms\|_W^2 + \frac{1}{2} \|U^\top U-V^\top V\|_F^2 + Q(U,V),
\label{eqn:asymmetricobj}
\end{equation}
where
\[
Q(U,V) = \lambda_1 \sum_{i=1}^{n_1} (\normtwo{U_i} - \alpha_1)_+^4 +\lambda_2\sum_{i=1}^{n_2} (\normtwo{V_i} - \alpha_2)_+^4.
\]

Here $\lambda_1,\lambda_2,\alpha_1,\alpha_2$ are parameters to choose. The notation $x_+$ represents $\max\{x,0\}$. Intuitively, the regularizer $Q(U,V)$ ensures $U,V$ have bounded row norms. The additional term $\frac{1}{2}\|U^\top U-V^\top V\|_F^2$ is a regularizer that is popular for asymmetric matrix completion~\citep{park2016non}. \cite{GeJZ17} showed the following result:

\begin{theorem}[\cite{GeJZ17}, Informal] Under appropriate settings of parameters $\lambda, \alpha$, if \\$p \ge \mbox{poly}(r,\mu,\kappa^\star,\log n)/n$ for some fixed polynomial, then with high probability, all local minima of Equation \eqref{eqn:asymmetricobj} are globally optimal, and they satisfy $UV^\top = \Ms$.
\end{theorem}

\subsection{Graph Laplacians}
\label{sec:laplacian}
For a weighted undirected graph $G = (V,E,w)$ with $n$ vertices and edge weights $w_e \ge 0$, let $D$ be a diagonal matrix containing the weighted degree of each vertex ($D_{i,i} = \sum_{(i,j)\in E} w_{(i,j)}$).
Let $A$ be the adjacency matrix of $G$ ($A_{i,j} = A_{j,i} = w_{(i,j)}$).
The Laplacian matrix of $G$ is a symmetric $n \times n$ matrix $L$ defined as $L = D - A$.
In other words, if we orient every edge $e = (i,j) \in E$ arbitrarily and represent it by a vector $b_e \in \R^n$ with $b_e(i) = 1$ and $b_e(j) = -1$.
The Laplacian matrix $L$ is equal to $L = \sum_{e\in E} w_e b_e b_e^\top$.

We use $L^{1/2}$ to denote the principal square root of a PSD matrix $L$.
Abusing notation, we use $L^{-1}$ for the Moore-Penrose inverse of $L$, and $L^{-1/2}$ for $(L^{-1})^{1/2}$.
The normalized Laplacian is the matrix $D^{-1/2} L D^{-1/2}$.
The effective resistance of an edge $e$ in a graph with Laplacian $L$ is $b_e^\top L^{-1} b_e$.
For any Laplacian $L$, we have $L^{-1} L = I_{\mathrm{im}(G)}$, where $I_{\mathrm{im}(G)} = I - \frac{1}{n} J$ is the identity matrix on the image space of $L$.
We often abbreviate $I_{\mathrm{im}(G)}$ as $I$.

We say two graphs are {\em spectrally similar} if the following holds: 

\begin{definition}[Spectral Similarity~\citep{SpielmanT11}]
\label{def:spectral}
Suppose $L_1, L_2$ are the Laplacians for graphs $G_1, G_2$ respectively, we say $G_1$ and $G_2$ are $\epsilon$-spectrally similar if and only if
\[
(1-\epsilon) L_1 \preceq  L_2 \preceq (1+\epsilon)L_1.
\]
\end{definition}

Graph sparsification is known to be a special case of sparsifying sum of rank-one PSD matrices (see, e.g., \citep{BatsonSS12}).
Similarly, we can reduce Problem~\ref{prob:main} to Problem~\ref{prob:identity}.

\begin{problem}
\label{prob:identity}
For a set of $m$ vectors $\{v_i\}_{i=1}^m$, assume there exists a subset $S$ of vectors such that $(1-\beta)I\preceq \sum_{i \in S} v_i v_i^\top \preceq (1+\beta)I$. Compute a set of weights $w_i \ge 0$ such that
\[(1-O(\beta)-\epsilon) I \preceq \sum_{i} w_i v_i v_i^\top \preceq I.\]
\end{problem}

To see why Problem~\ref{prob:main} can be reduced to Problem~\ref{prob:identity}, let $L_H$ be the Laplacian for the complete bipartite graph. The reduction simply sets $v_e = (L_H)^{-1/2} b_e$ for each edge $e\in E$.

%
%
%


\section{Pre-Processing: Reweighting the Entries}
\label{sec:bss}

In this section, we present a nearly-linear time algorithm for Problem~\ref{prob:main}.
As we discussed in Section~\ref{sec:laplacian}, Problem~\ref{prob:main} is equivalent to the problem of approximating the identity matrix (Problem~\ref{prob:identity}). We prove the following theorem for Problem~\ref{prob:identity}:

\begin{theorem}[Our Preprocessing Algorithm]
\label{thm:bss}
Fix $0 \le \eps, \beta \le 1/10$, and a graph Laplacian $L \in \R^{n \times n}$.
Given a set of $m$ vectors $\{v_i\}_{i=1}^m$, where each $v_i = L^{-1/2} b_i$ for some $b_i$ representing an edge (with only two non-zero entries, one $+1$ and one $-1$).
Assume there exist weights $w_i \ge 0$ such that~\footnote{We assume $0<\beta\le\frac{1}{10}$ is given, because we can do a binary search by running our algorithm and see if it succeeds.
It is worth mentioning that we never explicitly compute any $v_i = L^{-1/2} b_i$. 
See Appendix~\ref{app:bss} for more details.}
\[
(1-\beta) I \preceq \sum_{i=1}^m w_i v_i v_i^\top \preceq (1+\beta) I.
\]
We can find a set of weights $\tilde w_i \ge 0$ in $\tilde O(m / \eps^{O(1)})$ time, such that with high probability,
\[
(1-O(\beta)-\eps)I \preceq \sum_{i=1}^m \tilde w_i v_i v_i^\top \preceq I.
\]
\end{theorem}

We adapt techniques from recent developments on linear-sized graph sparsification~\citep{BatsonSS12, AllenLO15, LeeS17}.
The main difference between our problem and the graph sparsification problem is the following: instead of assuming $\sum_i v_i v_i^\top = I$, we only know the \emph{existence} of an unknown set $S$ such that \mbox{$\sum_{i\in S} v_i v_i^\top = I$}.
This prevents us from using some of the well-known techniques in graph sparsification (e.g., sampling by effective resistance \cite{SpielmanS11, LeeS15}).
For the same reason, any simple reweighting algorithms that are oblivious to whether a good set $S$ exists will not work.

One of our main contributions is to identify that the framework of \cite{BatsonSS12} is not only limited to graph sparsification.
The fact that the algorithm picks edges \emph{deterministically} makes it much more powerful,
and the analysis only requires the \emph{existence} of a ``good'' edge to add in each iteration.
On the technical level, our work departs from previous works in two important ways: (1) our algorithm works even when the hidden set $S$ has sum only \emph{approximately} equal to $I$; and (2) our analysis is considerably simpler, partly because we do not require the output weights to be sparse. 

We first give an overview of the framework of~\cite{BatsonSS12}.
We will maintain two barrier values $\ell < u$, and a weighted sum of the rank-one matrices $A = \sum_{i=1}^m w_i v_i v_i^\top$ such that $\ell I \prec A \prec u I$.
The plan is to start with some constants $\ell < 0 < u$, $A = 0$, and gradually increase the weights $\{w_i\}_i$, $u$ and $\ell$, while making sure that $A$ stays between the two barriers $u I$ and $\ell I$.
If we can increase $u$ and $\ell$ at roughly the same rate, the condition number of $A$ will become smaller.

Our approach in this section is most directly inspired by the recent work of~\cite{LeeS17}.
We use the following potential function proposed in \citep{LeeS17} to measure how far $A$ is from the barriers (both $uI$ and $\ell I$):
\begin{align*}
\Phi_{u,\ell}(A) & = \Phi_u(A) + \Phi_\ell(A), \text{ where} \\
\Phi_u(A) & = \tr \exp \left((u I - A)^{-1}\right), \\
\Phi_\ell(A) & = \tr \exp \left((A - \ell I)^{-1}\right).
\end{align*}
If $A$ is far from $u I$ and $\ell I$, then all eigenvalues of $uI - A$ and $A - \ell I$ are large and $\Phi_{u,\ell}(A)$ is small.
The potential function is going to guide us on how to increase the weights $w_i$ so that $A$ stays away from the barriers.
The derivatives of the potential functions with respect to $A$ are
\begin{align*}
\nabla \Phi_u (A) & = \exp \left((u I - A)^{-1}\right) (u I - A)^{-2}, \\ 
\nabla \Phi_\ell (A) & = - \exp \left((A - \ell I)^{-1}\right) (A - \ell I)^{-2}.
\end{align*}

For notational convenience, we define $C_{-} = \nabla \Phi_u (A)$, $C_{+} = -\nabla \Phi_\ell (A)$, and $C = C_+ - C_-$.
Note that when $\ell I \prec A \prec u I$, both $C_+$ and $C_-$ are PSD matrices.
The first order approximation of the potential function is
\begin{align*}
\Phi_{u,\ell}(A+\Delta)
 & \approx \Phi_{u,\ell}(A) + \nabla \Phi_{u}(A) \bullet \Delta + \nabla \Phi_{\ell}(A) \bullet \Delta \\
 & = \Phi_{u,\ell}(A) + C_- \bullet \Delta - C_+ \bullet \Delta
 = \Phi_{u,\ell}(A) - C \bullet \Delta.
\end{align*}

We want $\Phi_{u,\ell}(A+\Delta)$ to be small, which guarantees that $A+\Delta$ is far away from $\ell I$ and $u I$.
Therefore, in each iteration, we seek a matrix $\Delta$ such that
\begin{enumerate}
\item[(1)] $\Delta$ is small enough for the first-order approximation of $\Phi_{u,\ell}(A+\Delta)$ to be accurate; and
\item[(2)] $\Delta$ maximizes $C \bullet \Delta$, the reduction to (first-order approximation of) the potential function.
\end{enumerate}

Formally, let $\rho = (\lambda_{\min}\{u I - A, A - \ell I\})^2$.
When $0 \preceq \Delta \preceq \eps \rho I$, the first-order approximation of $\Phi_{u,\ell}(A+\Delta)$ is accurate (see Lemma~\ref{lem:potential-FO} in Appendix~\ref{app:bss}).
We are interested in the following SDP:
\begin{lp}
\label{eqn:sdp-oracle}
\maxi {C \bullet X}
\st \con{X \preceq \eps \rho I, \quad X = \sum_{i=1}^m x_i v_i v_i^\top \text{ (which implies $0 \preceq X$)},}
\end{lp}

Ideally, we would like to have $X = \eps \rho I$, and increase the barrier values by $\delta_{u}=\delta_{\ell}=\eps\rho$. If we can do this, $A$ grows equally in each dimension, the upper and lower barriers increase at the same rate, and the potential function remains unchanged: $\Phi_{u+\eps\rho,\ell+\eps\rho}(A+\eps\rho I) = \Phi_{u,\ell}(A)$.
While this is too good to be true, we will show that we can find an $X$ that is almost as good. 

We give a full description of our algorithm in Algorithm~\ref{alg:ls17}.

\begin{algorithm}
  \caption{Find $A = \sum_i w_i v_i v_i^\top \approx I$.}
  \label{alg:ls17}
  \SetAlgoVlined
  \SetKwInOut{Input}{Input}
  \Input{$\{v_i\}_{i=1}^m$, $\eps \le 1/10$.}
  $j = 0$, $A_0 = 0$, $\ell_0 = -\frac{1}{4}$, $u_0 = \frac{1}{4}$\; 
  \While{$u_j - \ell_j \le 1$}{
   Let $\rho \in [1-\eps, 1] \cdot (\lambda_{\min}\{u_j I - A_j, A_j - \ell I_j\})^2$\;
   Let $\Delta_j$ be an approximate solution to the SDP \eqref{eqn:sdp-oracle} with $C = -\left(\nabla \Phi_{u_j}(A_j)+\nabla \Phi_{\ell_j}(A_j)\right)$\;
   $A_{j+1} = A_j + \Delta_j$\;
   $\delta_{u,j}=\frac{\eps\rho}{2} \cdot \frac{(1+\beta+5\eps)}{1-2\eps}$, $\delta_{\ell,j}=\frac{\eps\rho}{2}\cdot\frac{(1-\beta-5\eps)}{1+2\eps}$\;
   $u_{j+1} = u_j + \delta_{u,j}$, $\ell_{j+1} = \ell_j + \delta_{\ell,j}$; \, $j = j + 1$\;
  }
  \Return{$A_j / u_j$}\;
\end{algorithm}

We will use the following lemmas (Lemmas~\ref{lem:sdp-sol}~and~\ref{lem:phi-no-increase}) to analyze Algorithm~\ref{alg:ls17} and prove Theorem~\ref{thm:bss}.
Lemma~\ref{lem:sdp-sol} shows that the SDP in \eqref{eqn:sdp-oracle} admits a good solution, and we can solve it approximately in nearly-linear time.
Lemma~\ref{lem:phi-no-increase} says that the potential function $\Phi_{u_j,\ell_j}(A_j)$ never increases, which guarantees that $A_j$ is far away from both $u_j I$ and $\ell_j I$ for all $j$.


\begin{lemma}
\label{lem:sdp-sol}
Fix $0 < \beta,\eps \le 1/10$.
In any iteration $j$ of Algorithm~\ref{alg:ls17}, given $A_j = \sum_{i=1}^m w_i v_i v_i^\top$ (implicitly by the weights $\{w_i \ge 0\}_{i=1}^m$) and corresponding barrier values $\ell = \ell_j$ and $u = u_j$,
\begin{enumerate}
\item[(1)] We can compute $\rho \in [1-\eps, 1] \cdot (\lambda_{\min}\{u I - A, A - \ell I\})^2$ w.h.p. in time $\tilde O(m / \eps^{O(1)})$.
\item[(2)] Let $C, C_{+}, C_{-}$ be defined as above. 
We can compute a set of weights $\{x_i \ge 0\}_{i=1}^m$ in time $\tilde O(m / \eps^{O(1)})$, such that w.h.p. for $X = \sum_{i=1}^m x_i v_i v_i^\top$,
\[
C \bullet X \ge \frac{\eps \rho}{2} \left((1-\beta-\eps)\tr(C_+) - (1+\beta+\eps)\tr(C_-)\right).
\]
\end{enumerate} 

\end{lemma}

\begin{lemma}
\label{lem:phi-no-increase}
Fix $0 < \beta, \eps \le 1/10$.
Let $A_{j+1} = A_j + \Delta_j$ denote the matrix in the $j$-th iteration of Algorithm~\ref{alg:ls17}.
If $\Delta_j$ is an approximate solution to the SDP that satisfies Lemma~\ref{lem:sdp-sol}, then we have $\Phi_{u_{j+1},\ell_{j+1}}(A_{j+1}) \le \Phi_{u_j,\ell_j}(A_{j})$.
\end{lemma}

We defer the proofs of these two lemmas to Appendix~\ref{app:bss}. Now we are ready to prove Theorem~\ref{thm:bss}.

\begin{proof}[Proof of Theorem~\ref{thm:bss}]
First, we show that $(1-O(\beta+\eps)) I \preceq A_j / u_j \preceq I$.
The condition number of $A_j$ is upper bounded by
$
\frac{u_j}{\ell_j} = \left(1-\frac{u_j-\ell_j}{u_j}\right)^{-1},
$
hence it suffices to show that $\frac{u_j-\ell_j}{u_j} = O(\beta + \eps)$.
Since $u_j - \ell_j > 1$ when the algorithm terminates,
\[
\frac{u_j - \ell_j}{u_j} < \frac{2(u_j - \ell_j) - 1}{u_j - \frac{1}{4}} = 2 \frac{(u_j - u_0) - (\ell_j - \ell_0)}{u_j - u_0}
  \le 2 \max_j\frac{\delta_{u,j} - \delta_{\ell,j}}{\delta_{u,j}} = O(\beta + \eps).
\]

Next, we analyze the running time of Algorithm~\ref{alg:ls17}.
The initial value of the potential function is $\Phi_{u_0, \ell_0}(0) = 2 \tr \exp (I) = 2n$.
By Lemma~\ref{lem:phi-no-increase} and a union bound over $j$, we have that with high probability, $\Phi_{u_j, \ell_j}(A_j) \le 2n$ for all $j \le O(\log n/\eps^2)$.
To see that $A_j$ must be far away from the barriers, consider only the contribution of $\lambda_{\min}(u_j I - A_j )$ to the potential function:
\[
2n \ge \Phi_{u_j}(A_j) = \tr \exp\left((u_j I - A_j)^{-1}\right) \ge \exp \left(\lambda_{\min}(u_j I - A_j)^{-1} \right).
\]
It follows that $\lambda_{\min}(u_j I - A) = \Omega(\log^{-1} n)$, and similarly $\lambda_{\min}(A_j - \ell_j I) = \Omega(\log^{-1} n)$.
Therefore, we know that $\rho \ge (1-\eps) (\lambda_{\min}\{u_j I - A_j, A_j - \ell I_j\})^2 = \Omega(\log^{-2} n)$, and $\delta_{u,j} - \delta_{\ell,j} = \Omega(\eps^2 \cdot \log^{-2} n)$ for all $j$.
Since the algorithm starts with $u_0 > \ell_0$ and terminates when $u_j - \ell_j > 1$, the number of iterations is at most $\frac{1}{\min_j (\delta_{u,j} - \delta_{\ell,j})} = O\left(\frac{\log^2 n}{\eps^2} \right)$.

It remains to show that each iteration takes nearly-linear time.
We maintain the matrices $A_j$ and $\Delta_j$ implicitly by the corresponding sets of weights, and add their weights together to get $A_{j+1}$. 
The input and output of the SDP are also represented implicitly by the weights.
By Lemma~\ref{lem:sdp-sol}, we can compute $\rho$ and find a near-optimal solution to the SDP in \eqref{eqn:sdp-oracle} in time $\tilde O(m / \eps^{O(1)})$.
The overall running time is $\tilde O(\frac{\log^2 n}{\eps^2} \cdot \frac{m}{\eps^{O(1)}}) = \tilde O(m / \eps^{O(1)})$.
\end{proof}


\subsection{From Graphs to Weight Matrices: Consequences of Spectral Similarity}
In this section, 
we show that the weights computed by Algorithm~\ref{alg:ls17} are useful for the semi-random matrix completion problem.
More specifically, we prove the following corollary of Theorem~\ref{thm:bss}.
Corollary~\ref{cor:preprocess} provides the spectral property that is crucial to our analysis of non-convex matrix-completion algorithms in Section~\ref{sec:matcomp}. 

\begin{corollary}
\label{cor:preprocess}
Fix $\beta > 0$.
Consider the matrix completion problem with ground truth $\Ms \in \R^{n_1 \times n_2}$.
There exists $p = O\left(\frac{\log n}{n_1 \beta^2}\right)$ such that if every entry of $\Ms$ is observed with probability at least $p$,
then w.h.p., we can compute a weight matrix $W \in \R^{n_1 \times n_2}$ in time $\tilde O(m / \beta^{O(1)})$, such that $W$ is supported on the observed entries, $\norminf{W} \le n_2$, $\normone{W} \le n_1$, and $\norm{W-J} = O(\beta \sqrt{n_1 n_2})$.
\end{corollary}
Recall that $J$ is the all ones matrix, $n = n_1 + n_2$, and we assume $n_1 \le n_2$.

Corollary~\ref{cor:preprocess} follows from Lemma~\ref{lem:random-graph}, Theorem~\ref{thm:bss}, and Lemma~\ref{lem:Lclose-Aclose}.
Lemma~\ref{lem:random-graph} provides concentration bounds for random matrices, which implies that when $p$ is large enough, the semi-random input contains a good subset of observations.
We can then apply Theorem~\ref{thm:bss} to show that our preprocessing algorithm can recover a good set of weights (i.e., a weighted graph that is spectrally similar to the complete bipartite graph).
Finally, Lemma~\ref{lem:Lclose-Aclose} shows that the closeness in the Laplacians of two graphs implies the closeness in their (normalized) adjacency matrices. 

\begin{lemma}
\label{lem:random-graph}
Let $G$ denote the $n_1 \times n_2$ complete bipartite graph.
We write $n = n_1 + n_2$ for the number of vertices, and $m = n_1 n_2$ for the number of edges.
Let $H$ denote a random graph generated by including each edge of $G$ independently with probability $p$.
W.h.p, we can re-weight edges in $H$ so that the Laplacian matrix $L_H$ is $\eps$-spectrally similar with $L_G$, where $\eps = O\left(\sqrt{\frac{n \log n}{p m}}\right)$.
\end{lemma}
\begin{proof}
For complete bipartite graphs, all edges have the same effective resistance, so uniform sampling among all the edges will produce a good spectral sparsifier.

Formally, we can use the main result of~\cite{SpielmanS11}: Fix any $0 < \eps < 1$. For sufficiently large $n$ and all graphs on $n$ vertices, there is a universal constant $C$ so that sampling $C n \log n/\eps^2$ edges independently (with sample probability $p_e$ proportional to [edge weight $\times$ effective resistant]) produces an $\eps$-spectral sparsifier with high probability.
The lemma allows reweighting on $H$ because when we include an edge we give it weight $1/p_e$.

At the core of~\citep{SpielmanS11} are matrix concentration inequalities~\citep{RudelsonV07, AhlswedeW02, Tropp12}.
Note that the original proof in~\citep{SpielmanS11} used sampling with replacement and holds only with constant probability, but the analysis can be adapted to show that sampling by effective resistance without replacement works with high probability.
\end{proof}


\begin{lemma}
\label{lem:Lclose-Aclose}
Let $L = D - A$ and $\tilde L = \tilde{D} - \tilde{A}$ be two graph Laplacians, where $D$ is the degree matrix and $A$ is the adjacency matrix of the graph.
If $(1-\eps) L \preceq \tilde L \preceq L$, then we have
\begin{enumerate}
\item[(1)] $(1-\eps) D_{i,i} \le \tilde{D}_{i,i} \le D_{i,i}$.
\item[(2)] $\norm{D^{-1/2}(\tilde{A} - A)D^{-1/2}} \le 3\eps$.
\end{enumerate}
\end{lemma}
\begin{proof}
For (1), the spectral similarity between $L$ and $\tilde L$ implies that $(1-\eps) x^\top L x \le x^\top \tilde{L} x \le x^\top L x$ for all $x \in \R^n$.
In particular, this holds for all standard basis vectors, so $(1-\eps) D_{i,i} \le \tilde{D}_{i,i} \le D_{i,i}$.

For (2), we know that $0 \preceq L - \tilde{L} \preceq \eps L$ and similarly $0 \preceq D-\tilde{D} \preceq \eps D$, and therefore
\begin{align*}
\norm{D^{-1/2}(\tilde{A} - A)D^{-1/2}}
& = \norm{D^{-1/2}(\tilde{D} - D + L - \tilde{L})D^{-1/2}} \\
& \le \norm{D^{-1/2} (D - \tilde D) D^{-1/2}} + \norm{D^{-1/2} (L - \tilde L) D^{-1/2}} \\
& \le \eps \norm{I} + \eps \norm{D^{-1/2} L D^{-1/2}} \le 3\eps. 
\end{align*}
The last step uses the fact that eigenvalues of a normalized Laplacian matrix $D^{-1/2} L D^{-1/2}$ are always between $0$ and $2$.
\end{proof}

We are now ready to prove Corollary~\ref{cor:preprocess} using Theorem~\ref{thm:bss} and Lemmas~\ref{lem:random-graph}~and~\ref{lem:Lclose-Aclose}.
\begin{proof}[Proof of Corollary~\ref{cor:preprocess}]
Let $G$ denote the $n_1 \times n_2$ complete bipartite graph ($n_1 \le n_2$).
Let $H$ be the graph corresponds to the entries revealed randomly, and let $H'$ denote the graph after the adversary added extra edges.
By Lemma~\ref{lem:random-graph}, for $p = O\left(\frac{\log n}{n_1 \beta^2}\right)$, with high probability, there exists edge weights for $H$ such that $(1-\beta) L_G \preceq L_H \preceq (1+\beta) L_G$.

Because the edges of $H'$ is a superset of the edges of $H$, there exist edge weights for $H'$ such that the same condition holds.
Since the vectors $\{L_G^{-1/2} b_e\}_{e \in H'}$ satisfy the condition in Problem~\ref{prob:identity}, we can invoke Theorem~\ref{thm:bss} with $\eps = \beta$ to obtain a set of weights $\{\tilde w_e\}_{e \in H'}$ such that $(1-O(\beta))L_G \preceq \sum_{e \in H'} \tilde w_e b_e b_e^\top \preceq L_G$ in time $\tilde O(m / \beta^{O(1)})$, where $m$ is the number of edges in $H'$.

Let $A$ denote the adjacency matrix of $G$, and let $A'$ denote the adjacency matrix of $H'$ with weights $\tilde w_e$.
Since both $A'$ and $A_G$ include only edges in the complete bipartite graph, we can write
\[
A = \left(\begin{array}{cc} 0 & J \\ J^\top & 0 \end{array}\right), \qquad
A' = \left(\begin{array}{cc} 0 & W \\ W^\top & 0 \end{array}\right), \qquad
\]
where $J$ is the all ones matrix and $W \in \R^{n_1 \times n_2}$ contains the edge weights $\tilde w_e$ ($W_{i,j} = \tilde w_e$ for every $(i,j) \in H'$, and $W_{i,j} = 0$ otherwise).
By Lemma~\ref{lem:Lclose-Aclose}, the row sum of $W$ is at most $n_2$ for every row, and the column sum of $W$ is at most $n_1$ for every column.
Again by Lemma~\ref{lem:Lclose-Aclose}, $\norm{D^{-1/2} (A' - A_{G}) D^{-1/2}} \le O(\beta)$, which implies that $\norm{W - J} = \norm{A - A'} = O(\beta \sqrt{n_1 n_2})$.
\end{proof}


\section{Application to Matrix Completion}
\label{sec:matcomp}
Most analysis of matrix completion relies on the fact that the observed entries are sampled uniformly at random. Let $W_{i,j} = 1/p$ if entry $(i,j)$ is observed. This assumption is mostly used to prove {\em concentration inequalities} related to the norm of low-rank matrices $\|M\|_W^2$.  In particular, the following two lemmas are used in most papers.

Lemma~\ref{lem:tangent} shows that for an $M$ that is in the ``tangent space'' (the linear space of $\Us X^\top + Y(\Vs)^\top$), the norm of $Z$ is preserved after we restrict to the observed entries. 
Lemma~\ref{lem:Delta_mc} shows that the norm is preserved for every incoherent matrix $XY^\top$.

\begin{lemma}[\cite{recht2011simpler}]
\label{lem:tangent}
Suppose $\Ms = \Us\Sigs(\Vs)^\top$. Suppose entries are revealed with probability $p$ independently, weight matrix $W_{i,j} = 1/p$ if $(i,j)$ is revealed and $0$ otherwise. 
For any $0 < \delta < 1$, when $p \ge \Omega(\frac{\mu r}{\delta^2 n} \log n)$, with high probability over the randomness of $W$ we have
\[
|\|M\|_W - \|M\|_F| \le \delta \|M\|_F,
\]
for any matrix $M\in \R^{n_1\times n_2}$ of the form $M = \Us X^\top + Y(\Vs)^\top$.
\end{lemma}

\begin{lemma}
\label{lem:Delta_mc}
Let $W$ be a random matrix where $W_{i,j} = 1/p$ with probability $p$, and $W_{i,j} = 0$ otherwise.
There exist universal constants $c_1$ and $c_2$, so that for any $\delta>0$, if $p \ge c_1  \frac{\log n}{\delta^2 n_1}$, then with probability at least $1-\frac{1}{2}n^{-4}$, we have for any matrices $X \in \R^{n_1\times r}, Y\in\R^{n_2\times r}$:
\begin{equation*}
\norm{XY^\top}^2_{W}  \le (1+\delta) \norm{X}_F^2\norm{Y}_F^2 + c_2 \sqrt{\frac{n}{p}}\norm{X}_F\norm{Y}_F\cdot \max_i\norm{X_i} \cdot \max_j\norm{Y_j}.
\end{equation*}
\end{lemma}

However, neither of these lemmas is applicable in our semi-random setting, because the weight matrix $W$ is no longer chosen randomly by nature.

Lemma~\ref{lem:tangent} is used in both the convex analysis (e.g.,~\citep{recht2011simpler}) and many local analyses for the non-convex methods (e.g.,~\citep{sun2015guaranteed}). Deterministic versions of Lemma~\ref{lem:tangent} include Assumption A2 in \citep{BhojanapalliJ14} and Assumption A3 in \citep{LiLR16}. Unfortunately, we do not know whether Assumption A2 in \citep{BhojanapalliJ14} is true even for random matrices, and we cannot guarantee Assumption A3 in \citep{LiLR16} because it is a condition that depends on the (unknown) ground truth.

Since we do not know how to obtain a deterministic version of Lemma~\ref{lem:tangent}, we turn our attention to Lemma~\ref{lem:Delta_mc}.
Lemma~\ref{lem:Delta_mc} is only used in more recent non-convex analyses (e.g.,~\citep{sun2015guaranteed, GeJZ17}).
We replace Lemma~\ref{lem:Delta_mc} with the following (stronger version of the) lemma, which states that if $W$ is close to the all ones matrix $J$ (which is guaranteed by our preprocessing algorithm), then the norm of $\|XY^\top\|_F$ is preserved after we weight the entries by $W$.
Recently, \citep{chen2017memory} has independently obtained a deterministic inequality similar to Lemma~\ref{lem:deterministc_main}.

\begin{lemma}[Preserving the Norm via Spectral Properties]
\label{lem:deterministc_main}
For any matrices $X\in \R^{n_1\times r}$, $Y \in \R^{n_2\times r}$, and $W \in \R^{n_1 \times n_2}$, we have
\[
|\|XY^\top\|_W^2 - \|XY^\top\|_F^2| \le \|W-J\| \|X\|_F\|Y\|_F\max_{i}\|X_i\|\max_{i}\|Y_i\|.
\]
\end{lemma}

\begin{proof}
Recall that $X_i$ is the $i$-th row of $X$, and $X\odot X \in \R^{n_1 \times r^2}$ is the Katri-Rao product.
We have 
\[\inner{(X\odot X)_i, (Y\odot Y)_j} = \inner{X_i\otimes X_i, Y_j\otimes Y_j} = (X_i^\top Y_j)^2 = (X Y^\top)_{i,j}^2.\]

As a result, we know $\|XY^\top\|_W^2 = \inner{XY^\top, XY^\top}_W = \tr((X\odot X)^\top W (Y\odot Y))$, and $\|XY^\top\|_F^2 = \inner{XY^\top, XY^\top} = \tr((X\odot X)^\top J (Y\odot Y))$.

We can also bound the Frobenius norm of the two product by: $\|X\odot X\|_F \le \|X\|_F\max_{i}\|X_i\|$, $\|Y\odot Y\|_F \le \|Y\|_F\max_{i}\|Y_i\|$. Therefore,
\begin{align*}
|\inner{XY^\top, XY^\top}_W - \inner{XY^\top, XY^\top}|
& = |\tr((X\odot X)^\top (W-J) (Y\odot Y))|\\
& \le \|X\odot X\|_F \|(W-J) (Y\odot Y)\|_F \\
& \le \|W-J\| \|X\odot X\|_F \|(Y\odot Y)\|_F  \\
& \le \|W-J\| \|X\|_F\|Y\|_F\max_{i}\|X_i\|\max_{i}\|Y_i\|. \qedhere 
\end{align*}
\end{proof}

Using this lemma, as well as techniques in \citep{GeJZ17}, we can prove the following theorem (see Appendix~\ref{app:matrix} for its proof):
%
\begin{theorem} \label{thm:asymmetric_local}
For matrix completion problem with ground truth $\Ms \in \R^{n_1 \times n_2}$, let $\mu, r, \sigs_1, \kappa^\star$ be the incoherence parameter, rank, largest singular value and condition number of $\Ms$.
Fix any error parameter $0 < \eps < 1$.
Suppose the weight matrix $W \in \R^{n_1 \times n_2}$ satisfies $\norminf{W} \le n_2$, $\normone{W} \le n_1$, and $\|W-J\| \le \frac{\eps c\sqrt{n_1n_2}}{\mu^3 r^3 (\kappa^\star)^3}$ for a small enough universal constant $c$.
Let $\alpha_1^2 = \frac{C\mu r\sigs_1}{n_1},\alpha_2^2 = \frac{C\mu r\sigs_1}{n_2}$, $\lambda_1 = \frac{C^2 n_1}{\mu r\kappa^\star}$, and $\lambda_2 = \frac{C^2 n_1}{\mu r\kappa^\star}$ for some large enough universal constant $C$. Then, any local minimum $(U, V)$ of the asymmetric Objective~\eqref{eqn:asymmetricobj} satisfies $\|UV^\top-\Ms\|_F^2 \le \epsilon \|\Ms\|_F^2$. 
\end{theorem}


Our main result (Theorem~\ref{thm:main}) follows immediately from Corollary~\ref{cor:preprocess} and Theorem~\ref{thm:asymmetric_local}.

We can choose $\beta = O(\eps / (\mu^3 r^3 (\kappa^\star)^3))$ in Corollary~\ref{cor:preprocess} so that our preprocessing algorithm produces a weight matrix $W$ that satisfies the requirements of Theorem~\ref{thm:asymmetric_local}.
By Theorem~\ref{thm:asymmetric_local}, the non-convex objective with weight matrix $W$ has no bad local optima.
The pre-processing time is $\tilde O(m / \beta^{O(1)}) = \tilde O(m \cdot \mathrm{poly}(\mu, r, \kappa^\star, \eps^{-1}))$.


\section{Conclusions}

In this paper, we showed that even though non-convex approaches for matrix completion are not robust in the semi-random model, but it is possible to fix them using a pre-processing step. The pre-processing step solves a few convex programs (packing SDPs) to ameliorate the influence of the semi-random adversary. 
Unlike the full convex relaxation for matrix completion, our pre-processing step runs in nearly-linear time. Combining our pre-processing step with non-convex optimization gives an algorithm that is robust in the semi-random model, and at the same time enjoys the efficiency of the non-convex approaches.

Our pre-processing step solves a variant of the graph sparsification problem. Given a graph $G$ formed by adding extra edges to $H$ (or a graph similar to $H$), we can produce a weighted version of $G$ that is spectrally similar to $H$. We believe this subroutine can be useful in other problems.

An immediate open problem is whether we can prove the output of the pre-processing step allows non-convex optimization to recover the ground truth {\em exactly}. This would require proving stronger concentration inequalities like Lemma~\ref{lem:tangent} using deterministic conditions. More broadly, we hope this work will inspire new ideas that make non-convex optimization more robust.

\section*{Acknowledgments} This work is supported by NSF CCF-1704656. We thank Qingqing Huang, Andrej Risteski, Srinadh Bhojanapalli, Yin Tat Lee for discussions at various stages of the work. Yu Cheng is also supported in part by NSF CCF-1527084, CCF-1535972, CCF-1637397, IIS-1447554, and NSF CAREER Award CCF-1750140.

\bibliographystyle{plainnat}
\bibliography{names,conferences,spectral,matrix_ref,semi_random}

\begin{thebibliography}{62}
\providecommand{\natexlab}[1]{#1}
\providecommand{\url}[1]{\texttt{#1}}
\expandafter\ifx\csname urlstyle\endcsname\relax
  \providecommand{\doi}[1]{doi: #1}\else
  \providecommand{\doi}{doi: \begingroup \urlstyle{rm}\Url}\fi

\bibitem[Agarwal et~al.(2016)Agarwal, Allen-Zhu, Bullins, Hazan, and
  Ma]{agarwal2016finding}
Naman Agarwal, Zeyuan Allen-Zhu, Brian Bullins, Elad Hazan, and Tengyu Ma.
\newblock Finding approximate local minima for nonconvex optimization in linear
  time.
\newblock \emph{arXiv preprint arXiv:1611.01146}, 2016.

\bibitem[Ahlswede and Winter(2002)]{AhlswedeW02}
Rudolf Ahlswede and Andreas~J. Winter.
\newblock Strong converse for identification via quantum channels.
\newblock \emph{IEEE Transactions on Information Theory}, 48\penalty0
  (3):\penalty0 569--579, 2002.

\bibitem[{Allen Zhu} et~al.(2015){Allen Zhu}, Liao, and Orecchia]{AllenLO15}
Zeyuan {Allen Zhu}, Zhenyu Liao, and Lorenzo Orecchia.
\newblock Spectral sparsification and regret minimization beyond matrix
  multiplicative updates.
\newblock In \emph{Proc. 46th ACM Symp. on Theory of Computing}, pages
  237--245, 2015.

\bibitem[{Allen Zhu} et~al.(2016){Allen Zhu}, Lee, and Orecchia]{AllenLO16}
Zeyuan {Allen Zhu}, Yin~Tat Lee, and Lorenzo Orecchia.
\newblock Using optimization to obtain a width-independent, parallel, simpler,
  and faster positive {SDP} solver.
\newblock In \emph{Proc. 27th ACM-SIAM Symp. on Discrete Algorithms}, pages
  1824--1831, 2016.

\bibitem[Arora et~al.(2015)Arora, Ge, Ma, and Moitra]{rgDict2}
Sanjeev Arora, Rong Ge, Tengyu Ma, and Ankur Moitra.
\newblock Simple, efficient and neural algorithms for sparse coding.
\newblock In \emph{Proc. 28th Conference on Learning Theory}, page 113–149,
  2015.

\bibitem[Batson et~al.(2012)Batson, Spielman, and Srivastava]{BatsonSS12}
Joshua~D. Batson, Daniel~A. Spielman, and Nikhil Srivastava.
\newblock Twice-{Ramanujan} sparsifiers.
\newblock \emph{SIAM Journal on Computing}, 41\penalty0 (6):\penalty0
  1704--1721, 2012.

\bibitem[Bengio et~al.(2013)Bengio, Courville, and Vincent]{deepsurvey1}
Yoshua Bengio, Aaron Courville, and Pascal Vincent.
\newblock Representation learning: A review and new perspectives.
\newblock \emph{IEEE Transactions on Pattern Analysis and Machine
  Intelligence}, 35\penalty0 (8):\penalty0 1798--1828, 2013.

\bibitem[Bhojanapalli and Jain(2014)]{BhojanapalliJ14}
Srinadh Bhojanapalli and Prateek Jain.
\newblock Universal matrix completion.
\newblock In \emph{Proc. 31st Intl. Conf. on Machine Learning}, pages
  1881--1889, 2014.

\bibitem[Blum and Spencer(1995)]{blum1995coloring}
Avrim Blum and Joel Spencer.
\newblock Coloring random and semi-random $k$-colorable graphs.
\newblock \emph{Journal of Algorithms}, 19\penalty0 (2):\penalty0 204--234,
  1995.

\bibitem[Cand{\`e}s and Recht(2009)]{candes2009exact}
Emmanuel~J Cand{\`e}s and Benjamin Recht.
\newblock Exact matrix completion via convex optimization.
\newblock \emph{Foundations of Computational mathematics}, 9\penalty0
  (6):\penalty0 717--772, 2009.

\bibitem[Cand{\`e}s and Tao(2010)]{candes2010power}
Emmanuel~J Cand{\`e}s and Terence Tao.
\newblock The power of convex relaxation: Near-optimal matrix completion.
\newblock \emph{Information Theory, IEEE Transactions on}, 56\penalty0
  (5):\penalty0 2053--2080, 2010.

\bibitem[Chen and Li(2017)]{chen2017memory}
Ji~Chen and Xiaodong Li.
\newblock Memory-efficient kernel {PCA} via partial matrix sampling and
  nonconvex optimization: a model-free analysis of local minima.
\newblock \emph{arXiv preprint arXiv:1711.01742}, 2017.

\bibitem[Chen and Wainwright(2015)]{chen2015fast}
Yudong Chen and Martin~J. Wainwright.
\newblock Fast low-rank estimation by projected gradient descent: General
  statistical and algorithmic guarantees.
\newblock \emph{arXiv preprint}, 1509.03025, 2015.

\bibitem[Cheng et~al.(2015)Cheng, Cheng, Liu, Peng, and Teng]{ChengCLPT15}
Dehua Cheng, Yu~Cheng, Yan Liu, Richard Peng, and Shang{-}Hua Teng.
\newblock Efficient sampling for {Gaussian} graphical models via spectral
  sparsification.
\newblock In \emph{Proc. 28th Conference on Learning Theory}, pages 364--390,
  2015.

\bibitem[Cohen et~al.(2014)Cohen, Kyng, Miller, Pachocki, Peng, Rao, and
  Xu]{CohenKMPPRX14}
Michael~B. Cohen, Rasmus Kyng, Gary~L. Miller, Jakub~W. Pachocki, Richard Peng,
  Anup~B. Rao, and Shen~Chen Xu.
\newblock Solving {SDD} linear systems in nearly $m \log^{1/2} n$ time.
\newblock In \emph{Proc. 45th ACM Symp. on Theory of Computing}, pages
  343--352, 2014.

\bibitem[Feige and Kilian(2001)]{feige2001heuristics}
Uriel Feige and Joe Kilian.
\newblock Heuristics for semirandom graph problems.
\newblock \emph{Journal of Computer and System Sciences}, 63\penalty0
  (4):\penalty0 639--671, 2001.

\bibitem[Feige and Krauthgamer(2000)]{feige2000finding}
Uriel Feige and Robert Krauthgamer.
\newblock Finding and certifying a large hidden clique in a semirandom graph.
\newblock \emph{Random Structures and Algorithms}, 16\penalty0 (2):\penalty0
  195--208, 2000.

\bibitem[Ge et~al.(2015)Ge, Huang, Jin, and Yuan]{ge2015escaping}
Rong Ge, Furong Huang, Chi Jin, and Yang Yuan.
\newblock Escaping from saddle points---online stochastic gradient for tensor
  decomposition.
\newblock \emph{arXiv preprint arXiv:1503.02101}, 2015.

\bibitem[Ge et~al.(2016)Ge, Lee, and Ma]{ge2016matrix}
Rong Ge, Jason~D Lee, and Tengyu Ma.
\newblock Matrix completion has no spurious local minimum.
\newblock In \emph{Proc. 28th Advances in Neural Information Processing
  Systems}, pages 2973--2981, 2016.

\bibitem[Ge et~al.(2017)Ge, Jin, and Zheng]{GeJZ17}
Rong Ge, Chi Jin, and Yi~Zheng.
\newblock No spurious local minima in nonconvex low rank problems: {A} unified
  geometric analysis.
\newblock In \emph{Proc. 34th Intl. Conf. on Machine Learning}, pages
  1233--1242, 2017.

\bibitem[Hardt(2014)]{hardt2014understanding}
Moritz Hardt.
\newblock Understanding alternating minimization for matrix completion.
\newblock In \emph{Proc. 55th IEEE Symp. on Foundations of Computer Science}.
  IEEE, 2014.

\bibitem[Hardt and Wootters(2014)]{hardt2014fast}
Moritz Hardt and Mary Wootters.
\newblock Fast matrix completion without the condition number.
\newblock In \emph{Proc. 27th Conference on Learning Theory}, pages 638--678,
  2014.

\bibitem[Hastie et~al.(2014)Hastie, Mazumder, Lee, and Zadeh]{hastie2014matrix}
Trevor Hastie, Rahul Mazumder, Jason Lee, and Reza Zadeh.
\newblock Matrix completion and low-rank {SVD} via fast alternating least
  squares.
\newblock \emph{Journal of Machine Learning Research}, 2014.

\bibitem[Jain et~al.(2013)Jain, Netrapalli, and Sanghavi]{jain2013low}
Prateek Jain, Praneeth Netrapalli, and Sujay Sanghavi.
\newblock Low-rank matrix completion using alternating minimization.
\newblock In \emph{Proc. 44th ACM Symp. on Theory of Computing}, pages
  665--674. ACM, 2013.

\bibitem[Jain and Yao(2011)]{JainY11}
Rahul Jain and Penghui Yao.
\newblock A parallel approximation algorithm for positive semidefinite
  programming.
\newblock In \emph{Proc. 52nd IEEE Symp. on Foundations of Computer Science},
  pages 463--471, 2011.

\bibitem[Jerrum(1992)]{Jerrum92}
Mark Jerrum.
\newblock Large cliques elude the metropolis process.
\newblock \emph{Random Struct. Algorithms}, 3\penalty0 (4):\penalty0 347--360,
  1992.

\bibitem[Jin et~al.(2017)Jin, Ge, Netrapalli, Kakade, and
  Jordan]{jin2017escape}
Chi Jin, Rong Ge, Praneeth Netrapalli, Sham~M. Kakade, and Michael~I. Jordan.
\newblock How to escape saddle points efficiently.
\newblock \emph{arXiv preprint arXiv:1703.00887}, 2017.

\bibitem[Kelner et~al.(2013)Kelner, Orecchia, Sidford, and {Allen
  Zhu}]{KelnerOSZ13}
Jonathan~A. Kelner, Lorenzo Orecchia, Aaron Sidford, and Zeyuan {Allen Zhu}.
\newblock A simple, combinatorial algorithm for solving {SDD} systems in
  nearly-linear time.
\newblock In \emph{Proc. 44th ACM Symp. on Theory of Computing}, pages
  911--920, 2013.

\bibitem[Keshavan et~al.(2010{\natexlab{a}})Keshavan, Montanari, and
  Oh]{keshavan2010matrix}
Raghunandan~H Keshavan, Andrea Montanari, and Sewoong Oh.
\newblock Matrix completion from a few entries.
\newblock \emph{IEEE Transactions on Information Theory}, 56\penalty0
  (6):\penalty0 2980--2998, 2010{\natexlab{a}}.

\bibitem[Keshavan et~al.(2010{\natexlab{b}})Keshavan, Montanari, and
  Oh]{keshavan2010matrixnoisy}
Raghunandan~H Keshavan, Andrea Montanari, and Sewoong Oh.
\newblock Matrix completion from noisy entries.
\newblock \emph{The Journal of Machine Learning Research}, 11:\penalty0
  2057--2078, 2010{\natexlab{b}}.

\bibitem[Koren(2009)]{koren2009bellkor}
Yehuda Koren.
\newblock The {Bellkor} solution to the {Netflix} grand prize.
\newblock \emph{Netflix prize documentation}, 81, 2009.

\bibitem[Koutis et~al.(2011)Koutis, Miller, and Peng]{KoutisMP11}
Ioannis Koutis, Gary~L. Miller, and Richard Peng.
\newblock A nearly-$m \log n$ time solver for {SDD} linear systems.
\newblock In \emph{Proc. 52nd IEEE Symp. on Foundations of Computer Science},
  pages 590--598, 2011.

\bibitem[Kucera(1995)]{Kucera95}
Ludek Kucera.
\newblock Expected complexity of graph partitioning problems.
\newblock \emph{Discrete Applied Mathematics}, 57\penalty0 (2-3):\penalty0
  193--212, 1995.

\bibitem[Kyng and Sachdeva(2016)]{KyngS16}
Rasmus Kyng and Sushant Sachdeva.
\newblock Approximate {G}aussian elimination for {L}aplacians - fast, sparse,
  and simple.
\newblock In \emph{Proc. 57th IEEE Symp. on Foundations of Computer Science},
  pages 573--582, 2016.

\bibitem[Lee and Sun(2015)]{LeeS15}
Yin~Tat Lee and He~Sun.
\newblock Constructing linear-sized spectral sparsification in almost-linear
  time.
\newblock In \emph{Proc. 56th IEEE Symp. on Foundations of Computer Science},
  pages 250--269, 2015.

\bibitem[Lee and Sun(2017)]{LeeS17}
Yin~Tat Lee and He~Sun.
\newblock An {SDP}-based algorithm for linear-sized spectral sparsification.
\newblock In \emph{Proc. 48th ACM Symp. on Theory of Computing}, pages
  678--687, 2017.

\bibitem[Li et~al.(2016)Li, Liang, and Risteski]{LiLR16}
Yuanzhi Li, Yingyu Liang, and Andrej Risteski.
\newblock Recovery guarantee of weighted low-rank approximation via alternating
  minimization.
\newblock In \emph{Proc. 33rd Intl. Conf. on Machine Learning}, pages
  2358--2367, 2016.

\bibitem[Makarychev et~al.(2012)Makarychev, Makarychev, and
  Vijayaraghavan]{makarychev2012approximation}
Konstantin Makarychev, Yury Makarychev, and Aravindan Vijayaraghavan.
\newblock Approximation algorithms for semi-random partitioning problems.
\newblock In \emph{Proc. 43rd ACM Symp. on Theory of Computing}, pages
  367--384. ACM, 2012.

\bibitem[Makarychev et~al.(2015)Makarychev, Makarychev, and
  Vijayaraghavan]{makarychev2015correlation}
Konstantin Makarychev, Yury Makarychev, and Aravindan Vijayaraghavan.
\newblock Correlation clustering with noisy partial information.
\newblock In \emph{Proc. 28th Conference on Learning Theory}, pages 1321--1342,
  2015.

\bibitem[Mathieu and Schudy(2010)]{mathieu2010correlation}
Claire Mathieu and Warren Schudy.
\newblock Correlation clustering with noisy input.
\newblock In \emph{Proc. 21st ACM-SIAM Symp. on Discrete Algorithms}, pages
  712--728, 2010.

\bibitem[Mazumder et~al.(2010)Mazumder, Hastie, and
  Tibshirani]{mazumder2010spectral}
Rahul Mazumder, Trevor Hastie, and Robert Tibshirani.
\newblock Spectral regularization algorithms for learning large incomplete
  matrices.
\newblock \emph{Journal of Machine Learning Research}, 11\penalty0
  (Aug):\penalty0 2287--2322, 2010.

\bibitem[Moitra et~al.(2016)Moitra, Perry, and Wein]{moitra2016robust}
Ankur Moitra, William Perry, and Alexander~S Wein.
\newblock How robust are reconstruction thresholds for community detection?
\newblock In \emph{Proc. 47th ACM Symp. on Theory of Computing}, pages
  828--841, 2016.

\bibitem[Negahban and Wainwright(2012)]{negahban2012restricted}
Sahand Negahban and Martin~J. Wainwright.
\newblock Restricted strong convexity and weighted matrix completion: Optimal
  bounds with noise.
\newblock \emph{Journal of Machine Learning Research}, 13\penalty0
  (May):\penalty0 1665--1697, 2012.

\bibitem[Park et~al.(2016)Park, Kyrillidis, Caramanis, and
  Sanghavi]{park2016non}
Dohyung Park, Anastasios Kyrillidis, Constantine Caramanis, and Sujay Sanghavi.
\newblock Non-square matrix sensing without spurious local minima via the
  {Burer}-{Monteiro} approach.
\newblock \emph{arXiv preprint arXiv:1609.03240}, 2016.

\bibitem[Peng and Spielman(2014)]{PengS14}
Richard Peng and Daniel~A. Spielman.
\newblock An efficient parallel solver for {SDD} linear systems.
\newblock In \emph{Proc. 45th ACM Symp. on Theory of Computing}, pages
  333--342, 2014.

\bibitem[Peng et~al.(2016)Peng, Tangwongsan, and Zhang]{PengTZ16}
Richard Peng, Kanat Tangwongsan, and Peng Zhang.
\newblock Faster and simpler width-independent parallel algorithms for positive
  semidefinite programming.
\newblock \emph{arXiv preprint arXiv:1201.5135v3}, 2016.

\bibitem[Perry and Wein(2017)]{perry2017semidefinite}
Amelia Perry and Alexander~S Wein.
\newblock A semidefinite program for unbalanced multisection in the stochastic
  block model.
\newblock In \emph{International Conference on Sampling Theory and Applications
  (SampTA)}, pages 64--67. IEEE, 2017.

\bibitem[Recht(2011)]{recht2011simpler}
Benjamin Recht.
\newblock A simpler approach to matrix completion.
\newblock \emph{Journal of Machine Learning Research}, 12:\penalty0 3413--3430,
  2011.

\bibitem[Rennie and Srebro(2005)]{rennie2005fast}
Jasson~DM Rennie and Nathan Srebro.
\newblock Fast maximum margin matrix factorization for collaborative
  prediction.
\newblock In \emph{Proc. 22nd Intl. Conf. on Machine Learning}, pages 713--719.
  ACM, 2005.

\bibitem[Rudelson and Vershynin(2007)]{RudelsonV07}
Mark Rudelson and Roman Vershynin.
\newblock Sampling from large matrices: An approach through geometric
  functional analysis.
\newblock \emph{Journal of the ACM}, 54\penalty0 (4):\penalty0 21, 2007.

\bibitem[Sa et~al.(2015)Sa, R{\'{e}}, and Olukotun]{DBLP:conf/icml/SaRO15}
Christopher~De Sa, Christopher R{\'{e}}, and Kunle Olukotun.
\newblock Global convergence of stochastic gradient descent for some non-convex
  matrix problems.
\newblock In \emph{Proc. 32nd Intl. Conf. on Machine Learning}, pages
  2332--2341, 2015.

\bibitem[Schmidhuber(2015)]{deepsurvey2}
J.~Schmidhuber.
\newblock Deep learning in neural networks: An overview.
\newblock \emph{Neural Networks}, 61:\penalty0 85--117, 2015.

\bibitem[Spielman and Srivastava(2011)]{SpielmanS11}
Daniel~A. Spielman and Nikhil Srivastava.
\newblock Graph sparsification by effective resistances.
\newblock \emph{SIAM Journal on Computing}, 40\penalty0 (6):\penalty0
  1913--1926, 2011.

\bibitem[Spielman and Teng(2011)]{SpielmanT11}
Daniel~A. Spielman and Shang{-}Hua Teng.
\newblock Spectral sparsification of graphs.
\newblock \emph{SIAM Journal on Computing}, 40\penalty0 (4):\penalty0
  981--1025, 2011.

\bibitem[Spielman and Teng(2014)]{SpielmanT14}
Daniel~A. Spielman and Shang{-}Hua Teng.
\newblock Nearly linear time algorithms for preconditioning and solving
  symmetric, diagonally dominant linear systems.
\newblock \emph{{SIAM} J. Matrix Analysis Applications}, 35\penalty0
  (3):\penalty0 835--885, 2014.

\bibitem[Srebro and Shraibman(2005)]{srebro2005rank}
Nathan Srebro and Adi Shraibman.
\newblock Rank, trace-norm and max-norm.
\newblock In \emph{Proc. 18th Conference on Learning Theory}, pages 545--560,
  2005.

\bibitem[Srebro et~al.(2004)Srebro, Rennie, and Jaakkola]{srebro2004maximum}
Nathan Srebro, Jason D.~M. Rennie, and Tommi~S. Jaakkola.
\newblock Maximum-margin matrix factorization.
\newblock In \emph{Proc. 16th Advances in Neural Information Processing
  Systems}, pages 1329--1336, 2004.

\bibitem[Sun and Luo(2015)]{sun2015guaranteed}
Ruoyu Sun and Zhi-Quan Luo.
\newblock Guaranteed matrix completion via nonconvex factorization.
\newblock In \emph{Proc. 56th IEEE Symp. on Foundations of Computer Science},
  pages 270--289. IEEE, 2015.

\bibitem[Tropp(2012)]{Tropp12}
Joel~A. Tropp.
\newblock User-friendly tail bounds for sums of random matrices.
\newblock \emph{Foundations of Computational Mathematics}, 12\penalty0
  (4):\penalty0 389--434, 2012.

\bibitem[Tu et~al.(2015)Tu, Boczar, Soltanolkotabi, and Recht]{tu2015low}
Stephen Tu, Ross Boczar, Mahdi Soltanolkotabi, and Benjamin Recht.
\newblock Low-rank solutions of linear matrix equations via {Procrustes} flow.
\newblock \emph{arXiv preprint arXiv:1507.03566}, 2015.

\bibitem[Zhao et~al.(2015)Zhao, Wang, and Liu]{zhao2015nonconvex}
Tuo Zhao, Zhaoran Wang, and Han Liu.
\newblock A nonconvex optimization framework for low rank matrix estimation.
\newblock In \emph{Proc. 27th Advances in Neural Information Processing
  Systems}, pages 559--567, 2015.

\bibitem[Zheng and Lafferty(2016)]{zheng2016convergence}
Qinqing Zheng and John Lafferty.
\newblock Convergence analysis for rectangular matrix completion using
  {Burer}-{Monteiro} factorization and gradient descent.
\newblock \emph{arXiv preprint arXiv:1605.07051}, 2016.

\end{thebibliography}

\appendix


\section{Counter Examples for Non-convex Approaches}
\label{app:examples}

In this section, we give counter-examples to some non-convex methods for matrix completion. For simplicity, we give examples for the {\em weighted} version of the non-convex objective. These examples can be translated to the semi-random adversary model using standard sampling techniques.

We will give the counter-examples in a simpler setting where $\Ms$ is known to be symmetric, we have $\Ms$ is an $n\times n$ matrix that can be decomposed as $\Ms = \Us(\Us)^\top$. In this case, we optimize:
\begin{equation}
\label{eqn:symmetricobj}
\min \; f(U) = \frac{1}{2}\|UU^\top - \Ms\|_W^2 + Q(U). 
\end{equation}
Here $Q(U)$ is the regularizer $\lambda \sum_{i=1}^n (\normtwo{U_i} - \alpha)_+^4$ (where $x_+ = \max\{x,0\}$). Parameters $\lambda, \alpha$ in the regularizer is specified later (see Lemma~\ref{lem:symmetricnormbound}).
%
Our examples also work for the asymmetric Objective~\eqref{eqn:asymmetricobj}.

\paragraph{Example Where Objective~\eqref{eqn:symmetricobj} Has Spurious Local Minimum.}

We first give an example where Objective~\eqref{eqn:symmetricobj} has a spurious local minimum. This is a simple rank 1 case where the intended solution $\Ms = u^\star (u^\star)^\top$ is the all ones matrix, and $u^\star = (1,1,...,1)^\top$ is the all ones vector.

In this example, all vectors will be represented by two blocks of size $n/2 \times 1$, and the value within each block will be the same; similarly all matrices will be partitioned into $2\times 2$ blocks (where each block has size $n/2\times n/2$), and entries within each block are the same.

For this example, we choose 
\[
u = \left(\begin{array}{c} \beta \\ -\beta \end{array}\right), \quad W = \left(\begin{array}{cc} \gamma J & J \\ J & \gamma J \end{array}\right),
\]
for any parameters $2^{-1/4} < \beta \le \frac{9}{10}$ and $\gamma = \frac{1+\beta^2}{1-\beta^2} < 10$.

\begin{lemma}
For the setting of $\Ms$, $u$, $W$ stated above, the objective function \eqref{eqn:symmetricobj} has a local minimum at $u$.
\end{lemma}

We will prove this lemma by second-order sufficient condition: gradient is 0 and Hessian is positive definite. Clearly $u$ is incoherent, so the incoherence regularizer does not matter. We first check the gradient of $f(u)$ is 0. This is due to our choice of $\gamma$ satisfies $\gamma(\beta^2-1) + (\beta^2+1) = 0$:
\[
\nabla f(u) = [(W + W^\top) * (uu^\top - \Ms)] u = 2[W * (uu^\top - \Ms)] u = 0.
\]
Next, we consider the Hessian of $f(u)$.
For any vector $\delta \in \R^n$, we know that
\[
\frac{1}{2}\delta^\top [\nabla^2 f(u)] \delta = \|\delta u^\top\|_W^2 + \inner{2uu^\top - \Ms,\delta\delta^\top}_W.
\]
We show this is strictly greater than 0 for all $\delta \neq 0$.
Let $\delta = \left( \begin{array}{c} \delta_1 \\ \delta_2 \end{array} \right)$ for $\delta_1, \delta_2 \in \R^{n/2}$.
Notice that $\|\delta u^\top\|_W^2$ is non-negative, therefore, we have
\begin{align*}
& \|\delta u^\top\|_W^2 + \inner{2uu^\top - \Ms,\delta\delta^\top}_W \\
& \ge \inner{2uu^\top - \Ms,\delta\delta^\top}_W \\
& = 2\beta^2 \left(\gamma\normone{\delta_1}^2 + \gamma\normone{\delta_2}^2 - 2 \normone{\delta_1} \normone{\delta_2} \right) - \left(\gamma\normone{\delta_1}^2 + \gamma\normone{\delta_2}^2 + 2 \normone{\delta_1} \normone{\delta_2}\right) \\
& \ge 2\beta^2 (\gamma-1) \left(\normone{\delta_1}^2 + \normone{\delta_2}^2\right) - (\gamma+1) \left(\normone{\delta_1}^2 + \normone{\delta_2}^2\right) \\
& = (2\beta^2 (\gamma-1) - (\gamma+1))\left(\normone{\delta_1}^2 + \normone{\delta_2}^2\right) > 0.
\end{align*}

The last step is due to $\delta \neq 0$ and our choice of $\beta$ and $\gamma$.

\paragraph{Example Where SVD Initialization Gives the Wrong Subspace.}

Many other non-convex methods depend on SVD to do initialization (e.g., \citep{jain2013low, hardt2014understanding}). Now we give an example where SVD cannot find the subspace correctly.

This is a rank-$2$ example.
For simplicity, all vectors are divided into blocks of size $n/4 \times 1$, and matrices are divided into blocks of size $n/4\times n/4$.
We write these as $4 \times 2$ or $4\times 4$ matrices, and they should be interpreted as blocks (constant multiplied by $J_{n/4\times 1}$ or $J_{n/4\times n/4}$).
The matrix $\mbox{Diag}(4,1)$ is a $2 \times 2$ diagonal matrix.
\[
\Ms = \left(\begin{array}{cc}1 & 1 \\ 1 & 1 \\ 1 & -1 \\ 1 & -1\end{array}\right) \mbox{Diag}(4,1)\left(\begin{array}{cccc}1 & 1 & 1 & 1 \\ 1 & 1 & -1 & -1\end{array}\right)
\]

\[
W = \left(\begin{array}{cccc}2 & 1 & 2 & 1 \\ 1 & 2 & 1 & 2 \\2 &1 & 2 & 1  \\ 1 & 2 & 1 & 2\end{array}\right); \qquad W*\Ms =  \left(\begin{array}{cccc}10 & 5 & 6 & 3 \\ 5 & 10 & 3 & 6 \\6 &3 & 10 & 5  \\ 3 & 6 & 5 & 10\end{array}\right)
\]

\begin{lemma}
Under the above setting for $\Ms, W$, the top two singular vectors of $W*\Ms$ are (up to normalization) $(1,1,1,1)$ and $(1,-1,1,-1)$. The principled angle between this subspace and the true subspace $\mbox{span}(\Ms)$ is 90 degrees.
\end{lemma}

This lemma is easy to verify numerically by computing the SVD of the $ 4\times 4$ matrix. The SVD of the original matrix follows the same block structure.

\paragraph{Converting Weighted Examples to Semi-Random Examples.}
In order to get counter examples in the semi-random model, pick a probability of observation $p$ (we need $p \ge \mbox{poly}(r) \log(n) /n$, and in our examples $p$ can be as large as $1/10$).
The semi-random adversary will reveal entry $(i,j)$ with probability $p_{i,j} = p W_{i,j} \ge p$ for the $W$ given in the examples.
This way, the expectation of the observed entries is exactly $W * \Ms$ (after scaling by $1/p$), and the expectation of the objective function is equal to $\|UU^\top - \Ms\|_W^2$.

For the first example, by standard concentration results, we know that the gradient and Hessian of the objective function are close to their expectations, so there is a local minimum near $u$.
For the second example, by standard random matrix theory, we know when $p$ is large enough the top singular space of the observed matrix is close to the top singular space of $W*\Ms$ (and thus far from the correct space).

\section{Omitted Proofs from Section~\ref{sec:bss}}
\label{app:bss}

In this section, we give more details about Section~\ref{sec:bss} and prove Lemmas~\ref{lem:sdp-sol}~and~\ref{lem:phi-no-increase}.

Recall that we are given a set of input vectors $\{v_i\}_{i=1}^m$ where each $v_i = L^{-1/2} b_i$ for a fixed Laplacian $L$ and some $b_i$ representing an edge, and we assume that there exist weights $w_i$ such that $(1-\beta) I \preceq \sum_{i=1}^m w_i v_i v_i^\top \preceq (1+\beta) I$.
The goal is to compute a set of weights $\tilde w_i$ in nearly-linear time so that $(1-O(\beta)-\eps)I \preceq \sum_{i=1}^m \tilde w_i v_i v_i^\top \preceq I$.

We maintain barrier values $u$ and $\ell$, and a weighted sum of the rank-one matrices $A = \sum_{i=1}^m \tilde w_i v_iv_i^\top$ such that $\ell I \prec A \prec uI$.
It is worth noting that we never explicitly compute the vectors $v_i = L^{-1/2} b_i$.  We store $A$ by keeping track of the weights $w_i$.
When updating the weights, we approximate the quantities $v_i^\top C v_i$ (for some matrix $C = C(A)$) simultaneously for all $i$ in nearly-linear time (see Lemma~\ref{lem:sdp-computation}).

We use the following potential functions proposed in \citep{LeeS17} to measure how far $A$ is away from the barriers:
\begin{align*}
\Phi_{u,\ell}(A) & = \Phi_u(A)+\Phi_{\ell}(A), \text{ where} \\
\Phi_u(A) & = \tr \exp \left((u I - A)^{-1}\right), \\
\Phi_\ell(A) & = \tr \exp \left((A - \ell I)^{-1}\right).
\end{align*}

The derivatives of the potential functions with respect to $A$ are
\begin{align*}
\nabla \Phi_u (A) & = \exp \left((u I - A)^{-1}\right) (u I - A)^{-2}, \\ 
\nabla \Phi_\ell (A) & = - \exp \left((A - \ell I)^{-1}\right) (A - \ell I)^{-2}.
\end{align*}

By convexity we have
\[
\Phi_{u,\ell}(A+\Delta) \ge \Phi_{u,\ell}(A) + \nabla \Phi_{u}(A) \bullet \Delta + \nabla \Phi_{\ell}(A) \bullet \Delta.
\]

The following lemma from \cite{LeeS17} shows that when $\Delta$ is small, the first-order approximation of $\Phi_{u,\ell}(A+\Delta)$ is a good estimation.

\begin{lemma}[\cite{LeeS17}]
\label{lem:potential-FO}
Let $A$ be a symmetric matrix.
Let $\ell < u$ be barrier values such that $u - \ell \le 1$ and $\ell I \prec A \prec u I$.
Assume that $0 \preceq \Delta$, $\Delta \preceq \eps (uI - A)^2$, and $\Delta \preceq \eps (A - \ell I)^2$ for $\eps \le 1/10$. Then,
\begin{alignat*}{3}
\Phi_{u}(A+\Delta) & \le \Phi_u(A) + (1+2\eps) \nabla \Phi_u(A) \bullet \Delta &&= \Phi_u(A) + (1+2\eps) C_- \bullet \Delta, \\
\Phi_{\ell}(A+\Delta) & \le \Phi_\ell(A) + (1-2\eps) \nabla \Phi_\ell(A) \bullet \Delta &&= \Phi_\ell(A) - (1-2\eps) C_+ \bullet \Delta; \\
\Phi_{u}(A-\Delta) & \le \Phi_u(A) - (1-2\eps) \nabla \Phi_u(A) \bullet \Delta &&= \Phi_u(A) - (1-2\eps) C_- \bullet \Delta, \\
\Phi_{\ell}(A-\Delta) & \le \Phi_\ell(A) - (1+2\eps) \nabla \Phi_\ell(A) \bullet \Delta &&= \Phi_\ell(A) + (1+2\eps) C_+ \bullet \Delta.
\end{alignat*}
\end{lemma}

Recall that for notational convenience, we write 
\begin{alignat*}{3}
C_+ & = -\nabla \Phi_\ell (A) &&= \exp \left((A - \ell I)^{-1}\right) (A - \ell I)^{-2}, \\
C_- & = \nabla \Phi_u (A) &&= \exp \left((u I - A)^{-1}\right) (u I - A)^{-2}.
\end{alignat*}

When $\ell I \prec A \prec u I$, both $C_+$ and $C_-$ are PSD matrices.
Recall that $\rho = (\lambda_{\min}\{uI-A, A-\ell I\} )^2$.
We are interested in the following packing SDP:
\begin{lp*}
\maxi {(C_+ - C_-) \bullet X} \tag{\ref{eqn:sdp-oracle}}
\st \con{X \preceq \eps \rho I, \quad X = \sum_{i=1}^m x_i v_i v_i^\top \text{ (which implies $0 \preceq X$)}.}
\end{lp*}

The constraint $X \preceq \eps \rho I$ implies that $X \preceq \eps(uI-A)^2$ and $X \preceq \eps(A-\ell I)^2$.
Thus, by Lemma~\ref{lem:potential-FO}, when $\eps \le 1/10$, the first-order approximation of the potential function is accurate:
\[
\Phi_{u,\ell}(X) - (C_+ - C_-) \bullet X \le \Phi_{u,\ell}(A+X) \le \Phi_{u,\ell}(A) - ((1-2\eps) C_+ - (1+2\eps)C_- ) \bullet X.
\]
Let $C = C_+ - C_-$. The SDP in \eqref{eqn:sdp-oracle} is naturally trying to find an $X$ to maximize the $C \bullet X$, while making sure $C \bullet X$ is a good approximation to the reduction of the potential function.

Ideally, we would like to have $X = \eps \rho I$ so that $A$ grows equally in each dimension, and the potential function stays the same:
\[
\Phi_{u+\eps\rho,\ell+\eps\rho}(A+\eps\rho I) = \Phi_{u,\ell}(A).
\]
When $X = \eps \rho I$, the objective value of the SDP is
\[
(C_+ - C_-) \bullet X = \eps \rho (\tr(C_+) - \tr(C_-)).
\]
While in general $I$ may not be in the span of the rank-one matrices, we will show in Lemma~\ref{lem:sdp-computation} that we can compute an $X$ with
\[
(C_+ - C_-) \bullet X \ge \frac{\eps \rho}{2} \left((1-\beta-2\eps)\tr(C_+) - (1+\beta+2\eps)\tr(C_-)\right).
\]
We first prove Lemma~\ref{lem:phi-no-increase}, which states that the potential function does not increase if $\Delta_j$ in each iteration satisfies Lemma~\ref{lem:sdp-computation}.

\begin{proof}[Proof of Lemma~\ref{lem:phi-no-increase}]
We want to show $\Phi_{u_{j+1},\ell_{j+1}}(A_{j+1}) \le \Phi_{u_{j},\ell_{j}}(A_{j})$, where $A_{j+1} = A_j + \Delta_j$.

When we increase the weights $\{w_i\}$ and expand $A$, the lower barrier potential $\Phi_\ell$ decreases (since $A$ gets farther away from $\ell I$) and the upper barrier $\Phi_u$ increases (since $A$ gets closer to $u I$).
When we increase the barrier values $u$ and $\ell$, the opposite happens: $\Phi_\ell$ increases and $\Phi_u$ decreases.
Intuitively, the proof works by carefully increasing $u$ and $\ell$ to cancel out the change due to adding $\Delta_j$, while making sure both $u$ and $\ell$ increase at roughly the same rate.

Recall that $C_+ = -\nabla \Phi_{\ell_j} (A_j)$ and $C_- = \nabla \Phi_{u_j} (A_j)$.
Formally, we have
\begin{align}
& \Phi_{u_{j+1},\ell_{j+1}}(A_{j+1}) \notag \\
& = \Phi_{u_j}(A_{j+1}-\delta_{u,j} I) + \Phi_{\ell_j}(A_{j+1}-\delta_{\ell,j} I) \notag \\
& \le \Phi_{u_j,\ell_j}(A_{j+1}) - (1-2\eps) \nabla \Phi_{u_j}(A_{j+1}) \bullet (\delta_{u,j} I) - (1+2\eps) \nabla \Phi_{\ell_j}(A_{j+1}) \bullet (\delta_{\ell,j} I) \tag{Lemma~\ref{lem:potential-FO}}\\
& \le \Phi_{u_j,\ell_j}(A_{j+1}) - (1-2\eps) \nabla \Phi_{u_j}(A_{j}) \bullet (\delta_{u,j} I) - (1+2\eps) \nabla \Phi_{\ell_j}(A_{j}) \bullet (\delta_{\ell,j} I) \label{eqn:nabla-Aj1} \\
& = \Phi_{u_j,\ell_j}(A_{j+1}) - (1-2\eps) \delta_{u,j} \tr(C_-) + (1+2\eps) \delta_{\ell,j} \tr(C_+) \label{eqn:phi-change-1}.
\end{align}
Our choice of $\delta_{u,j}, \delta_{\ell,j}$ satisfies that $\delta_{u,j}, \delta_{\ell,j} \le \eps \rho$ when $\eps, \beta \le 1/10$, which allows us to apply Lemma~\ref{lem:potential-FO}.
Step~\eqref{eqn:nabla-Aj1} uses $A_{j} \preceq A_{j+1}$ and the fact that, for any $\ell I \prec A_1 \preceq A_2 \prec u I$, we have $0 \preceq \nabla \Phi_u (A_1) \preceq \nabla \Phi_u (A_2)$ and $0 \preceq -\nabla \Phi_\ell (A_2) \preceq -\nabla \Phi_\ell (A_1)$.

We continue to bound the change of the potential function when we set $A_{j+1} = A_j + \Delta_j$ for any $\Delta_j$ that satisfies Lemma~\ref{lem:sdp-computation}.
\begin{align}
& \Phi_{u_j,\ell_j}(A_{j+1})
= \Phi_{u_j,\ell_j}(A_{j} + \Delta_j) \notag \\
& \le \Phi_{u_j, \ell_j}(A_j) + (1+2\eps) C_- \bullet \Delta_j - (1-2\eps) C_+ \bullet \Delta_j \tag{Lemma~\ref{lem:potential-FO}} \\
& = \Phi_{u_j, \ell_j}(A_j) - (C_+ - C_-) \bullet \Delta_j + 2\eps (C_+ + C_-) \bullet \Delta_j \notag \\
& = \Phi_{u_j, \ell_j}(A_j) - (C_+ - C_-) \bullet \Delta_j + 2\eps^2\rho \tr(C_+ + C_-) \tag{$\Delta_j \le \eps\rho I$} \notag \\
& \le \Phi_{u_j, \ell_j}(A_j) - \frac{\eps\rho}{2}\left[(1-\beta-\eps)\tr(C_+) - (1+\beta+\eps)\tr(C_+)\right] + 2\eps^2 \rho\tr(C_+ + C_-) \tag{Lemma~\ref{lem:sdp-computation}} \\
& = \Phi_{u_j, \ell_j}(A_j) + \frac{\eps\rho(1+\beta+5\eps)}{2} \tr(C_-) - \frac{\eps\rho(1-\beta-5\eps)}{2} \tr(C^+) \label{eqn:phi-change-2}.
\end{align}
We conclude the proof by comparing lines \eqref{eqn:phi-change-1} and \eqref{eqn:phi-change-2}, and setting $\delta_{u,j} = \frac{\eps\rho(1+\beta+5\eps)}{2(1-2\eps)}$ and $\delta_{\ell,j} = \frac{\eps\rho(1-\beta-5\eps)}{2(1+2\eps)}$.
The trace terms cancel out and we get $\Phi_{u_{j+1},\ell_{j+1}}(A_{j+1}) \le \Phi_{u_{j},\ell_{j}}(A_{j})$ as needed.
\end{proof}

We remark that the best we can hope for is to increase the upper and lower barriers at a rate of $1+\beta$ vs. $1-\beta$ (which is the case as $\eps \to 0$), because the ground-truth is a set of weights $w_i$ with $(1-\beta) I \preceq \sum_i w_i v_i v_i^\top \preceq (1+\beta) I$. Our algorithm only achieves a ratio of $1+O(\beta+\eps)$ vs. $1-O(\beta+\eps)$ for several reasons: (1) the error in the first-order approximation of the potential function, (2) we solve the SDP approximately, and (3) we use Taylor expansion and Johnson-Lindenstrauss to speed up the computation.

We break Lemma~\ref{lem:sdp-sol} into two lemmas and prove them separately.
Lemma~\ref{lem:sdp-sol-exist} states that SDP~\eqref{eqn:sdp-oracle} has a good solution. Lemma~\ref{lem:sdp-computation} shows that we can compute $\rho$ and solve this packing SDP~\eqref{eqn:sdp-oracle} approximately in nearly-linear time.

\begin{lemma}
\label{lem:sdp-sol-exist}
Let $A$ be a symmetric matrix. Let $\ell < u$ be barrier values such that $\ell I \prec A \prec u I$.
The SDP in \eqref{eqn:sdp-oracle} has a solution $X$ with
\[
C \bullet X \ge \eps \rho \left(\frac{1-\beta}{1+\beta}\tr(C_+) - \tr(C_-)\right).
\]
\end{lemma}
\begin{proof}
When $X = \eps \rho I$, the objective value is $C \bullet X = \eps \rho \tr(C)$.
Note that $C$ is not PSD, so when $X \approx \eps \rho I$, we need to split $C$ into the difference of two PSD matrices $C_+$ and $C_-$ to bound the error.

Recall that there exists a set of weights $w_i$ with $(1-\beta) I \preceq \sum_{i=1}^m w_i v_i v_i^\top \preceq (1+\beta)I$.
We look at a solution of this SDP with $X = \frac{\rho}{1+\beta} \sum_{i=1}^m w_i v_i v_i^\top$.
It follows directly that $X$ is feasible since $X$ is a weighted sum of $v_i v_i^\top$ and $X \preceq \eps \rho I$.
For the objective value, since $\frac{1-\beta}{1+\beta} \eps \rho I \preceq X \preceq \eps \rho I$,
\begin{equation*}
(C_+ - C_-) \bullet X
  \ge C_+ \bullet (\frac{1-\beta}{1+\beta} \eps \rho I) - C_- \bullet (\eps \rho I)
  = \eps\rho\left(\frac{1-\beta}{1+\beta} \tr(C_+) - \tr(C_-)\right). \qedhere 
\end{equation*}
\end{proof}

We now provide details about how to implement Algorithm~\ref{alg:ls17} in nearly-linear time.
We remark that similar implementations were shown in~\cite{AllenLO15, LeeS15, LeeS17}.

\begin{lemma}
\label{lem:sdp-computation}
Fix $0 < \beta,\eps \le 1/10$. 
Given an $n \times n$ matrix $A = \sum_{i=1}^m w_i v_i v_i^\top$ represented by a set of weights $\{w_i \ge 0\}_{i=1}^m$, let $\ell < u$ be barrier values such that $u - \ell \le 1$ and $(\ell + g) I \prec A \prec (u - g) I$ for some gap $g = \Omega(\log^{-2} n)$.

We can compute $\rho$ and weights $\{\tilde w_i\}_{i=1}^m$ in $\tilde O(m / \eps^5)$ time, such that with high probability,
\begin{enumerate}
\item[(1)] $\rho \in [1-\eps, 1] \cdot (\lambda_{\min}\{u I - A, A - \ell I\})^2$;
\item[(2)] $X = \sum_{i=1}^m \tilde x_i v_i v_i^\top$ satisfies $X \preceq \eps\rho I$ and
\[
(C_+ - C_-) \bullet X \ge \frac{\eps \rho}{2} \left((1-\beta-\eps)\tr(C_+) - (1+\beta+\eps)\tr(C_-)\right).
\]
\end{enumerate}
\end{lemma}
\begin{proof}
Recall that $v_i = L^{-1/2} b_i$ for a fixed Laplacian $L$.
Let $\hat L = \sum_i w_i b_i b_i^\top$ be the Laplacian specified by the weights of $A$.
In this proof, we will frequently use the following fact,
\[
A = \sum_{i=1}^m w_i v_i v_i^\top = L^{-1/2} \sum_{i=1}^m w_i b_i b_i^\top L^{-1/2} = L^{-1/2} \hat L L^{-1/2}.
\]

\begin{enumerate}
\item[(1)] We show how to compute $\rho \in [1-\eps, 1] \cdot \lambda_{\min}(u I - A)^2$. The approach is similar for $A - \ell I$.
It is sufficient to compute $\rho \approx_{\eps/2} \lambda_{\min}(u I - A)^2$.~\footnote{We write $a \approx_{\eps} b$ for $\exp(-\eps) a \le b \le \exp(\eps) a$. This extends naturally to PSD matrices, where $A \approx_{\eps} B$ means $\exp(-\eps)A \preceq B \preceq \exp(\eps) A$. It is sufficient to approximate $\rho$ up to a factor of $1\pm \exp(\eps/2)$ because $\exp(\eps) \le \frac{1}{1-\eps}$.}
By Lemma~\ref{lem:mat-taylor}, there exists a degree $\tilde O\left(\frac{\log(1/\eps)}{g}\right) = \tilde O(\log^2 n \log(1/\eps))$ polynomial $p(A)$ such that $p(A) \approx_{\eps/4} (u I - A)^{-2}$.
Since $\lambda_{\max} (p(A)) \approx_{\eps/4} \lambda_{\max} ((u I - A)^{-2}) = \left(\lambda_{\min} (u I - A)^2\right)^{-1}$, it is sufficient to approximate $\lambda_{\max} (p(A))$.

Observe that for any $n \times n$ PSD matrix $M$,
\[
\lambda_{\max}(M) \le \left(\tr\left(M^{2k}\right)\right)^{1/2k} \le n^{1/2k} \lambda_{\max} (M).
\]
In particular, for $k = O(\log n / \eps)$ we can get $\left(\tr(p(A)^{2k})\right)^{1/2k} \approx_{\eps/4} \lambda_{\max}(p(A))$,
and thus, we can return $\rho = \left(\tr(p(A)^{2k})\right)^{-1/2k} \approx_{\eps/2} \lambda_{\min}(u I - A)^2$.

It remains to show that we can approximate
\[
\tr(p(A)^{2k}) = \tr(p(L^{-1/2} \hat L L^{-1/2})^{2k}) = \tr(p(L^{-1} \hat L)^{2k}).
\]
Let $M = p(L^{-1} \hat L)^{k}$ so $\tr(p(A)^{2k}) = \tr(M^2)$.
We approximate each diagonal entry of $M^2$ by writing it as $\left(M^2\right)_{i,i} = \chi_i^\top M M \chi_i = \normtwo{M \chi_i}^2$, where $\chi_i$ denote the $i$-th standard basis vector.
By the Johnson-Lindenstrauss lemma, we can generate a random $O(\log n/\eps^2) \times n$ matrix $Q$, so that with high probability, for all $1 \le i \le n$,
\[
\normtwo{M \chi_i}^2 = \normtwo{Q M \chi_i}^2.
\]
Note that $Q M \chi_i$ is the $i$-th column of $QM$.
We can compute (approximately) $Q M = \left(M Q^\top\right)^\top$ by multiplying each column of $Q^\top$ through $M$.

This can be done in time $\tilde O(n/\eps^5)$, because $Q^\top$ has $O(\log n/\eps^2)$ columns, and matrix-vector multiplication with $M = p(L^{-1} \hat L)^{k}$ can be implemented using $k \cdot \mbox{deg}(p) = \tilde O(\log^3 n / \eps)$ matrix-vector multiplications with $L^{-1} \hat L$. We will show that matrix-vector multiplication with $L^{-1} \hat L$ can be done in time $\tilde O(n / \eps^2)$, so the overall running time is $\tilde O(n / \eps^5)$.

Recall that the number of edges in $\hat L$ is at most $m$.
Let $m'$ denote the number of edges in $L$.
W.l.o.g., we can assume both $m, m' = O(n/\eps^2)$ by sparsifying the input graphs first.
Therefore, one matrix-vector multiplication with $L^{-1} \hat L$ can be done in time $\tilde O(n/\eps^2)$, by first multiplying the vector through $\hat L$, and then solving a linear system in $L$ in $\tilde O(m' \log (1/\eps))$ time~\citep{SpielmanT14, KoutisMP11, KelnerOSZ13, PengS14, ChengCLPT15, CohenKMPPRX14, KyngS16}.

\item[(2)]
Since we represent the variable $X$ of the SDP by a set of weights $\{x_i\}_{i=1}^m$, the objective function is of the form $C \bullet X = C \bullet \left(\sum_{i=1}^m x_i v_i v_i^\top\right) = \sum_{i=1}^m x_i (v_i^\top C v_i)$.
Let $c \in \R^m$ be a vector with $c_i = v_i^\top C v_i$.
The SDP in \eqref{eqn:sdp-oracle} can be rewritten as
\begin{lp*}
\maxi {c^\top x}
\st \con{\sum_{i=1}^m x_i (v_i v_i^\top) \preceq \eps \rho I.}
\end{lp*}

This is a packing SDP that can be solved in polylogarithmic iterations (see, e.g., \citep{JainY11, AllenLO16, PengTZ16}).
Formally, we use Lemma~\ref{lem:alo16} from \cite{AllenLO16}.
Because Lemma~\ref{lem:alo16} returns a solution $X$ with $\expect{}{C \bullet X} \ge \frac{4}{5} \mathrm{OPT}$, it must return a $\frac{3}{5}$-approximation with probability at least $1/2$.
Since $\frac{3}{5} > \frac{1+\beta}{2}$, we can invoke Lemma~\ref{lem:alo16} $O(\log n)$ times so that we get $\frac{1+\beta}{2}$-approximation with high probability. We assume this event happens for the rest of the proof.

Let $c_i^+ = v_i^\top C_+ v_i$ and $c_i^- = v_i^\top C_- v_i$.
If we can approximate $c^+$ and $c^-$ by a (multiplicative) factor of $1 \pm \frac{\eps}{2}$, we have
\begin{align*}
(c^+ - c^-)^\top x
& \ge \frac{1+\beta}{2} \mathrm{OPT} - \frac{\eps}{2} (c^+ + c^-)^\top x \tag{Lemma~\ref{lem:alo16}} \\
& \ge \frac{1+\beta}{2}  \mathrm{OPT} - \frac{\eps^2 \rho}{2} \tr(C^+ + C^-) \tag{$X \preceq \eps \rho I$} \\
& \ge \frac{1+\beta}{2} \eps \rho \left(\frac{1-\beta}{1+\beta}\tr(C_+) - \tr(C_-)\right) - \frac{\eps^2 \rho}{2} \tr(C^+ + C^-) \tag{Lemma~\ref{lem:sdp-sol-exist}}. \\
& = \frac{\eps \rho}{2} \left( (1-\beta-\eps) \tr(C_+) - (1+\beta+\eps) \tr(C_-) \right).
\end{align*}

Finally, 
we will show how to approximate $c^-_i$ by a factor of $1 \pm \frac{\eps}{2}$ for all $1 \le i \le m$ in time $\tilde O(m / \eps^{O(1)})$.
The algorithms for approximating $c^+_i$ and implementing the oracle required by Lemma~\ref{lem:alo16} follow from a similar approach.\footnote{
We remark that the problem of approximating these quantities is akin to that of approximating (relative) effective resistances~\citep{SpielmanS11, AllenLO15, LeeS15}, and a nearly-linear time algorithm for computing the same quantities was shown in~\citep{LeeS17}.}

Recall that $C_- = \exp((uI - A)^{-1}) (uI-A)^{-2}$, where $A = L^{-1/2} \hat L L^{-1/2}$.
By Lemma~\ref{lem:mat-taylor} the assumption that $g = \Omega(\log^{-2} n)$, there exists a degree $\tilde O\left(\frac{\log(1/\eps)}{g^2}\right) = \tilde O(\log^4 n \log(1/\eps))$ polynomial $q(A)$ such that
\[
q(A) \approx_{\eps/6} \exp\left(\frac{1}{2}(uI - A)^{-1}\right) (uI-A)^{-1}.
\]
Because both sides are matrix polynomials of $A$, we can diagonalize them simultaneously so that the approximation only happens to the eigenvalues. Therefore,
\[
\left(q(A)\right)^2 \approx_{\eps/3} \exp((uI - A)^{-1}) (uI-A)^{-2},
\]
which implies $v_i^\top q(A)^2 v_i \in [1\pm \frac{\eps}{2}] v_i^\top C_- v_i$, since $\exp(\eps/3) \le 1+\eps/2$ when $\eps \le \frac{1}{10}$.

Recall that $m$, $m'$ denotes the number of edges in $\hat L$ and $L$.
Recall that $L = B^\top B$ where $B \in \R^{m' \times n}$ is the edge-vertex incident matrix of $L$.
Fix some $1 \le i \le m$.
Let $v_i = L^{-1/2} b_i$ where $b_i = \chi_u - \chi_{u'}$ for the $i$-th edge $(u, u')$.

For any $1 \le i \le m$, we have
\begin{align*}
v_i^\top C_- v_i
& \approx_{\eps/3} \normtwo{q(A) v_i}^2 \\
& = \normtwo{q(L^{-1/2} \hat L L^{-1/2}) L^{-1/2} b_i}^2 \\
& = \normtwo{L^{1/2} q(L^{-1} \hat L) L^{-1} b_i}^2 \\
& = \normtwo{B q(L^{-1} \hat L) L^{-1} b_i}^2 \\
& = \normtwo{B q(L^{-1} \hat L) L^{-1} (\chi_u - \chi_{u'})}^2.
\end{align*}
So the quantities $\{c^-_i\}_{i=1}^m$ are just the squared distances between the $m$-dimensional points $\{B q(L^{-1} \hat L) L^{-1} \chi_u\}_{u\in V}$.
We invoke the Johnson-Lindenstrauss lemma and generate a random $O(\log n/\eps^2) \times m$ matrix $Q$, so that with high probability, for all $1 \le u, u' \le n$,
\[
\normtwo{Bq(L^{-1} \hat L)L^{-1}(\chi_u - \chi_{u'})} \approx_{\eps/6} \normtwo{QBq(L^{-1} \hat L)L^{-1}(\chi_u - \chi_{u'})}.
\]
Recall that $A_i$ is the $i$-th row of a matrix $A$.
Let $Z = QBq(L^{-1} \hat L)L^{-1}$ and $Y = QBq(L^{-1} \hat L)$.
Both $Y$ and $Z$ have $O(\log n/\eps^2)$ rows and $n$ columns.
We have $Z^\top = L^{-1} Y^\top$, which allows us to approximate each $(Z_i)^\top = L^{-1} (Y_i)^\top$ by solving a linear system in $L$.
The time it takes to solve $O(\log n/\eps^2)$ linear systems in $L$ is $\tilde O(n/\eps^4)$, because we can assume $m, m' = O(n/\eps^2)$ by sparsifying the input graphs.

We can compute $Y^\top = q(L^{-1} \hat L) B^\top Q^\top$ in $\tilde O(n/\eps^4)$ time, since we can perform matrix-vector multiplication with $q(L^{-1} \hat L)$ in time $\tilde O(m/\eps^2)$, and $B^\top Q^\top$ can be computed in $\tilde O(n/\eps^4)$ time because $B$ has $2m' = O(n/\eps^2)$ non-zeros and $Q$ has $O(\log n/\eps^2)$ rows. 

Finally, because we can implement the oracle required by Lemma~\ref{lem:alo16} in time $\tilde O(n/\eps^4)$, we can compute a $\frac{3}{5}$-approximate solution to SDP \eqref{eqn:sdp-oracle} in time $\tilde O(n/\eps^4)$. Note that the number of iterations in Lemma~\ref{lem:alo16} does not depend on $\eps$ because we only need a $\frac{3}{5}$-approximation.
\qedhere 
\end{enumerate}
\end{proof}

The overall running time of Algorithm~\ref{alg:ls17} is $\tilde O(m / \eps^7)$, because there are at most $O(\log n / \eps^2)$ iterations (shown at the end of Section~\ref{sec:bss}), and each iteration can be implemented to run in time $\tilde O(m / \eps^{5})$ by Lemma~\ref{lem:sdp-computation}.

One of our main contributions is conceptual: we show that the framework of~\cite{BatsonSS12} can be applied to a much broader settings to obtain scalable algorithms.
On a technical level, because there exists a hidden set $S$ whose sum is only \emph{approximately} equal to $I$, the optimal solution to the SDP will be worse, so we need to carefully control the error caused by this, and move the barriers at a slightly different rate.
Our analysis is considerably simpler than that in~\citep{LeeS17}, partly because we do not require the output weights to be sparse;
We also take care of two minor issues with~\citep{LeeS17}: They assumed $\rho$ can be computed exactly for simplicity, and they proved Taylor expansion of $C_-$ can be truncated (where it should be $C_-^{1/2}$ as in Lemma~\ref{lem:mat-taylor}).

\begin{lemma}[\cite{AllenLO16}]
\label{lem:alo16}
Consider the following SDP with $M_i \in \R^{n \times n}$, $M_i \succeq 0$ and $c \in \R^m$:
\begin{lp*}
\maxi{c^\top x}
\st \con{x \ge 0}
    \con{\sum_{i=1}^m x_i M_i \preceq I.}
\end{lp*}
Suppose $c$ is given explicitly and we have access to $M_i$ via an oracle $\OO_{\eta, \delta}$ which on input $x \in \R^m$ outputs a vector $y \in \R^m$ such that
\[
y_i \in (1 \pm \frac{\delta}{2}) \; M_i \bullet \exp \left(\eta \Bigl(\sum_{i} x_i M_i - I \Bigr) \right)
\]
in time $T_{\eta, \delta}$ for any $x \in \R^m$ such that $x \ge 0$ and $\sum_{i=1}^m x_i M_i \preceq 2I$.
Then, we can output an $x$ in time $\tilde O(T_{\eta, \delta} \log^2 (m n) / \delta^3)$ such that
\[
\expect{}{c^\top x} \ge (1-O(\delta)) \cdot \mathrm{OPT} \text{ with } \sum_{i=1}^m x_i M_i \preceq I.
\]
\end{lemma}

\begin{lemma}
\label{lem:mat-taylor}
Let $A$ be a real symmetric matrix. When $(u-1) I \prec A \prec (u-g) I$ for some $0 < g < 1$, we can compute
\begin{enumerate}
\item[(1)] A polynomial $p(A)$ of degree $O\left(\frac{\log(1/(\eps g))}{g}\right)$ such that $p(A) \approx_{\eps} (uI-A)^{-2}$.
\item[(2)] A polynomial $q(A)$ of degree $O\left(\frac{\log(1/(\eps g))}{g^2}\right)$ such that $q(A) \approx_{\eps} \exp\left(\frac{1}{2}(uI - A)^{-1}\right) (uI-A)^{-1}$.
\end{enumerate}
\end{lemma}
\begin{proof}
The lemma is proved by truncating Taylor expansions.
Because $A$ is symmetric, the matrix polynomials $p(A)$ and $q(A)$ can be diagonalized simultaneously with $A$.
Therefore, it is sufficient to prove such polynomials exist for scalars.

\begin{enumerate}
\item[(1)] Let $f(x) = x^{-2}$. Let $p(a) = \hat p(u-a)$, and we define $\hat p(\cdot)$ to be the first $d$ terms of the Taylor expansion of $f(x)$ at $x=1$.
\[
f(x) = \sum_{i=0}^d \frac{f^{(i)}(1)}{i!}(x-1)^n + \frac{1}{d!} \int_1^{x} f^{(d+1)}(t) (x-t)^d dt.
\]
We know that all eigenvalues of $(uI-A)$ are in the interval $(g, 1)$.
For any $x \in (g, 1)$, there exists some $d = \left(g^{-1} \log(1/(\eps g)) \right)$ such that the remainder of the Taylor series satisfies
\begin{align*}
\left| f(x) - \sum_{i=0}^d \frac{f^{(i)}(1)}{i!}(x-1)^n \right|
  & = \left| \frac{1}{d!} \int_1^{x} f^{(d+1)}(t) (x-t)^d dt \right| \\
  & = \frac{(1-x)^{d+1} (1+x+dx)}{x^2} \\
  & \le (1-g)^{d+1}\frac{d+2}{g^2} \le \eps.
\end{align*}
\item[(2)] Let $h(x) = \exp\left(\frac{1}{2}x^{-1}\right) x^{-1}$.
Let $q(a) = \hat q(u-a)$, and we define $\hat q(\cdot)$ to be the first $d$ terms of the Taylor expansion of $h(x)$ at $x=1$.

For any $x \in (g, 1)$ and $t \in [x, 1]$, $h$ is holomorphic on a neighborhood of the ball $B := \{z \in \C : |z - t| \le r \}$ for $r = t - x/2$, so we can bound the coefficients of the Taylor expansion using Cauchy's estimates.
\[
\frac{1}{(d+1)!} h^{(d+1)}(t) \le r^{-d-1} \sup_{z \in B} |h(z)| \le r^{-d-1} \cdot 2 \exp(x^{-1}) x^{-1}.
\]
There exists some $d = O\left(g^{-2} \log(1/(\eps g))\right)$ such that the remainder at $x \in (g, 1)$ satisfies
\begin{align*}
\left| h(x) - \hat q(x) \right|
  & = \left| \frac{1}{d!} \int_1^{x} h^{(d+1)}(t) (x-t)^d dt \right| \\
  & \le 2(d+1) \exp(x^{-1}) x^{-1} \int_x^1 \frac{(t-x)^d}{(t-x/2)^{d+1}} dt \\
  & \le 4(d+1) \exp(x^{-1}) x^{-2} \int_x^1 \bigl(1-\frac{x}{2}\bigr)^d dt \\
  & \le 4(d+1) \exp(g^{-1}) g^{-2} \bigl(1-\frac{g}{2}\bigr)^d \le \eps. \qedhere 
\end{align*}
\end{enumerate}
\end{proof}


\section{Using Deterministic Conditions for Matrix Completion}
\label{app:matrix}
In this section, we prove Theorem~\ref{thm:asymmetric_local}.

We use Lemma~\ref{lem:deterministc_main} and the techniques in \citep{GeJZ17} to show that all local minima of the non-convex objective functions are close to the ground truth.
We first restate the objective functions: Equation~\eqref{eqn:symmetricobj} for the symmetric case and Equation~\eqref{eqn:asymmetricobj} for the asymmetric case.
\begin{align*}
\min \; f(U) &= \frac{1}{2}\|UU^\top - \Ms\|_W^2 + Q(U),  \tag{\ref{eqn:symmetricobj}} \\
\min \; f(U, V) &= 2\|UV^\top - \Ms\|_W^2 + \frac{1}{2} \|U^\top U-V^\top V\|_F^2 + Q(U,V), \tag{\ref{eqn:asymmetricobj}}
\end{align*}
where $x_+ = \max\{x,0\}$, $Q(U) = \lambda \sum_{i=1}^n (\normtwo{U_i} - \alpha)_+^4$, and $Q(U,V) = \lambda_1 \sum_{i=1}^{n_1} (\normtwo{U_i} - \alpha_1)_+^4 +\lambda_2\sum_{i=1}^{n_2} (\normtwo{V_i} - \alpha_2)_+^4$.

We start with an overview of the analysis in \citep{GeJZ17} in Appendix~\ref{app:matrix-overview}.
Because Lemma~\ref{lem:tangent} is no longer true in the semi-random setting, we cannot use the proof of \cite{GeJZ17} in a black-box way.
We will handle symmetric (Appendix~\ref{app:matrix-symmetric}) and asymmetric (Appendix~\ref{app:matrix-asymmetric}) cases separately.

\subsection{Overview of the Analysis in \citep{GeJZ17}}
\label{app:matrix-overview}

We give a brief overview of the techniques in \citep{GeJZ17}. The materials in this section are independent of the concentration bounds, so they remain valid in the semi-random model.

\paragraph{Measuring Distance between Matrices.} The first problem in analyzing Objective~\eqref{eqn:symmetricobj} is that the optimal solution is not unique: given a matrix $\Ms = \Us(\Us)^\top$, for any orthonormal matrix $R$ we also have $\Ms = (\Us R)(\Us R)^\top$. To take this symmetry into account, we define the distance between two matrices as follows:

\begin{definition}\label{def:difference}
Given matrices $U, \Us \in \R^{n\times r}$, their difference is defined to be $\Delta = U - \Us R$, where $R\in \R^{r\times r}$ is an $r\times r$ orthonormal matrix that minimizes $\|U - \Us R\|_F^2$.
\end{definition}

The benefit of this definition of distance is summarized in the following lemma:

\begin{lemma}[Lemma 6 in \citep{GeJZ17}]
\label{lem:normconnect}
Given matrices $U, \Us \in \R^{n\times r}$, let $M = UU^\top$ and $\Ms = \Us(\Us)^\top$, let $\Delta$ be the difference defined in Definition~\ref{def:difference}, then
\[
\|\Delta\Delta^\top\|_F^2 \le 2\|M - \Ms\|_F^2,
\]
and
\[
\sigs_r\|\Delta\|_F^2 \le \frac{1}{2(\sqrt{2}-1)} \|M - \Ms\|_F^2.
\]
\end{lemma}

The lemma states that when $\Delta$ is large, $M$ is also far from $\Ms$. This would not be true if we simply defined $\Delta = U - \Us$ without considering the best rotation of $\Us$.
From now on, we will always assume $\Us$ is {\em aligned} with $U$ in the sense that $R = I$ and $\Delta = U - \Us$ (this can be guaranteed by choosing the appropriate global optimum that $U$ is comparing to). 

\paragraph{Main Proof for the Symmetric Case.}
First, we introduce notations for the Hessian. The Hessian of $f(U)$ is a 4-th order tensor (because the variable $U$ is a matrix). For an $n\times r$ matrix $X$, we use $[\nabla^2 f(U)](X)$ to denote the quadratic form of the Hessian evaluated at $X$. The Hessian is positive semidefinite (PSD), iff $[\nabla^2 f(U)](X) \ge 0$ for every $X$.

The main idea of \cite{GeJZ17} is to focus on the direction of $\Delta$:
To prove $UU^\top = \Ms$, instead of using $\nabla f(U) = 0$ and $\nabla^2 f(U)$ is PSD, it is sufficient to work with $\inner{\nabla f(U), \Delta} = 0$ and $[\nabla^2 f(U)](\Delta) \ge 0$.
The next lemma, which is the main lemma in \citep{GeJZ17}, derives a particular inequality that is very useful in proving convergence.
Lemma~\ref{lem:gjzmain_sym} is proved by simplifying the second-order term $[\nabla^2 f(U)](\Delta)$ given that the first-order term $\inner{\nabla f(U), \Delta}$ is 0.

\begin{lemma}[Lemma 7 in \citep{GeJZ17}]
\label{lem:gjzmain_sym}
Let $M = UU^\top$ and $\Delta$ is the difference of $U$ and $\Us$ as in Definition~\ref{def:difference}, if $U$ is a local minimum of Objective \eqref{eqn:symmetricobj}, then
\[
0 \le [\nabla^2 f(U)](\Delta) = \|\Delta\Delta^\top\|_W^2 - 3\|M-\Ms\|_W^2 + ([\nabla^2 Q(U)](\Delta) -4\inner{\nabla Q(U),\Delta}).
\]
\end{lemma}

To see why this inequality is useful intuitively, assume the regularizer term is 0 (the current vector is incoherent so the incoherence regularizer is not active), and assume further the $W$-norms are very close to Frobenius norm (which is essentially guaranteed by Lemmas~\ref{lem:tangent}~and~\ref{lem:Delta_mc} when entries are observed randomly), then we have
\[
\|\Delta\Delta^\top\|_F^2 - 3\|M-\Ms\|_F^2 \ge 0.
\]
However, by Lemma~\ref{lem:normconnect} we know $\|\Delta\Delta^\top\|_F^2 \le 2\|M-\Ms\|_F^2$, so the only way this equation can hold is if $\|M-\Ms\|_F = 0$, and therefore, all local optima are global.

Finally, we state the lemma that shows the regularizer term is indeed small.\footnote{
The constant in Lemma~\ref{lem:extra_bound_symmetric} is slightly different from that of \citep{GeJZ17}, but it follows from the same proof by choosing a larger universal constant $C$.}

\begin{lemma}[Lemma 11 in \citep{GeJZ17}]
\label{lem:extra_bound_symmetric}
Let $U$ and $\Delta$ be defined as above.
Choose $\alpha^2 = \frac{C\mu r\sigs_1}{n}$ and $\lambda = \frac{C^2 n}{\mu r\kappa^\star}$ where $C$ is a large enough universal constant, then we have
\[
([\nabla^2 Q(U)](\Delta) -4\inner{\nabla Q(U),\Delta}) \le 0.1\sigs_r \|\Delta\|_F^2.
\]
\end{lemma}

\paragraph{Reduction from Asymmetric Case to the Symmetric Case.}

To handle asymmetric matrices, \citep{GeJZ17} gives a way to essentially reduce asymmetric matrices to symmetric matrices. 

For variables $U,V$ and optimal solution $\Us,\Vs$, we define the following matrices:
\[
Z = 
\begin{pmatrix}
 U\\
 V
\end{pmatrix}; \;
Z^\star = 
\begin{pmatrix}
 \Us\\
 \Vs
\end{pmatrix}; \quad
N = ZZ^\top; \; N^\star = (Z^\star)(Z^\star)^\top.
\]

In the asymmetric setting, we consider $\Delta=\begin{pmatrix}
 \Delta_U\\
 \Delta_V
\end{pmatrix}$ as the difference between $Z$ and $Z^\star$ as in Definition~\ref{def:difference}, and we also rotate $Z^\star$ so that $\Delta = Z-Z^\star$.

Roughly speaking, we want to design an objective function that reduces the asymmetric case to a symmetric matrix completion problem with variables $Z$ and ground truth $N^\star$.
This is impossible if we only focus on the term $2 \|UV^\top - \Ms\|_W^2$, because it does not depend on the diagonal blocks of $(N - N^\star)$.
Since we cannot observe the diagonal blocks of $N^\star$, we try to add a term so that the Hessian of $f(Z)$ acts like a block identity tensor on $N$.
The additional term $\frac{1}{2} \|U^\top U-V^\top V\|_F^2$ is introduced for exactly this purpose.
%

Let $Q(Z) = Q(U,V)$ be the same regularizer as in Objective~\eqref{eqn:asymmetricobj}. \cite{GeJZ17} proved the following lemma:

\begin{lemma}[Essentially Lemma 16 in \citep{GeJZ17}]
\label{lem:gjzmain_asym}
Let $Z$, $Z^\star$, $N$, $N^\star$, and $\Delta$ be defined as above, if $Z$ is a local minimum of Objective~\eqref{eqn:asymmetricobj}, then
\[
0 \le [\nabla^2 f(Z)](\Delta) \le \|\Delta\Delta^\top\|_{\bar{W}}^2 - 3\|N-N^\star\|_{\bar{W}}^2 + ([\nabla^2 Q(Z)](\Delta) -4\inner{\nabla Q(Z),\Delta}).
\]
where
\[
\bar{W} = \begin{pmatrix}
J & 2W - J \\
2W^\top - J & J
\end{pmatrix}.
\]
\end{lemma}

Similar to the symmetric case, Lemma~\ref{lem:gjzmain_asym} is proved by simplifying the second-order term $[\nabla^2 f(Z)](\Delta)$ given that the first-order term $\inner{\nabla f(Z), \Delta}$ is 0.

Notice that we have $\|\bar{W} - J\| = 2\|W-J\|$.
If our preprocessing algorithm guarantees that $W$ is close to $J$, then $\bar W$ is close to $J$ as well.

Finally, we have a corresponding lemma that shows the regularization term is small.

\begin{lemma}[Lemma 22 in \citep{GeJZ17}]\label{lem:extra_bound_asymmetric}
Let $Z$ and $\Delta$ be defined as above.
Choose $\alpha_1^2 = \frac{C\mu r\sigs_1}{n_1},\alpha_2^2 = \frac{C\mu r\sigs_1}{n_2}$ and $\lambda_1 = \frac{C^2 n_1}{\mu r\kappa^\star},\lambda_2 = \frac{C^2 n_1}{\mu r\kappa^\star}$ where $C$ is a large enough universal constant, then we have
\[
([\nabla^2 Q(Z)](\Delta) -4\inner{\nabla Q(Z),\Delta}) \le 0.1\sigs_r \|\Delta\|_F^2.
\]
\end{lemma}

\subsection{Proof of Our Symmetric Case}
\label{app:matrix-symmetric}
We first prove a variant of Lemma 9 in \citep{GeJZ17} in the semi-random model. Lemma~\ref{lem:symmetricnormbound} shows that any local minima of Objective~\eqref{eqn:symmetricobj} have bounded row norms.

\newcommand{\grad}{\nabla}
\newcommand{\poly}{\mbox{poly}}

\begin{lemma} \label{lem:symmetricnormbound}
When the weight matrix $W$ satisfies $\norminf{W} \le n$, choose $\alpha^2 = \frac{C\mu r\sigs_1}{n}$ and $\lambda = \frac{C^2 n}{\mu r\kappa^\star}$ where $C$ is a large enough universal constant. For Objective~\eqref{eqn:symmetricobj}, we have for any matrix $U$ with $\grad f (U) = 0$, 
\begin{equation*}
\max_{i}\normtwo{U_i}^2 = O\left(\frac{\mu^2 r^2 \kappa^\star \sigs_1}{n}\right).
\end{equation*}
\end{lemma}

\begin{proof}
Recall that $U_i \in \R^{1 \times r}$ is the $i$-th row of $U \in \R^{n \times r}$ and $e_i \in \R^{r \times 1}$ is the $i$-th standard basis vector.

The gradient $\nabla f(U)$ is equal to $2(W*(M-\Ms))U + \nabla Q(U)$, where
\[ \nabla Q(U) = 4\lambda \sum_{i=1}^n (\normtwo{U_i} - \alpha)_+^3 \frac{e_i U_i}{\normtwo{U_i}^2}. \]

Let $i^\star$ be the row index with the maximum $\ell_2$-norm, if $\normtwo{U_{i^\star}} \le 2\alpha$ then we are done. On the other hand, if $\normtwo{U_{i^\star}} > 2\alpha$, we will consider the gradient along $e_{i^\star}U_{i^\star}$. 
We have
\begin{align*}
0 & = \inner{\nabla f(U), e_{i^\star}U_{i^\star}} \\
& = \inner{e_{i^\star}^\top[2(W*(UU^\top - \Ms))U + \nabla Q(U)], U_{i^\star}} \\
& \ge 4\lambda (\normtwo{U_{i^\star}}-\alpha)_+^3\normtwo{U_{i^\star}} - 2\inner{e_{i^\star}^\top \Ms, e_{i^\star}^\top UU^\top}_W \\
& \ge \frac{\lambda}{2}\normtwo{U_{i^\star}}^4 - 2n \maxnorm{\Ms} \maxnorm{UU^\top} \\
& \ge \frac{\lambda}{2}\normtwo{U_{i^\star}}^4 - 2\mu r\sigs_1\normtwo{U_{i^\star}}^2.
\end{align*}

The third step removes the term $\inner{e_{i^\star}^\top U U^\top U, U_{i^\star}} = \fnorm{e_{i^\star}^\top U U^\top}^2 \ge 0$.
The fourth step uses that $\normtwo{U_{i^\star}} > 2\alpha$ and $\norminf{W} \le n$ (every row of $W$ has $\ell_1$-norm at most $n$).
The last step is due to $\maxnorm{UU^\top} = \normtwo{U_{i^\star}}^2$; and $\maxnorm{\Ms} = \max_{i,j} \inner{U_i, V_j} \le \max_{i,j} \normtwo{U_i} \normtwo{V_j} \le \frac{\sigs_1 \mu r}{n}$ because $\Ms$ is incoherent.
As a result, we know that $\normtwo{U_{i^\star}}^2 \le \frac{4\mu r\sigs_1}{\lambda} = O\left(\frac{\mu^2 r^2 \kappa^\star \sigs_1}{n}\right)$ by our choice of $\lambda$.
\end{proof}

Next, we will show that all local minima are close to the ground truth.

\begin{lemma}
Fix any error parameter $0 < \eps < 1$.
For a weight matrix $W$ such that $\norminf{W} \le n$ and $\|W-J\| \le \frac{\eps cn}{\mu^2 r^2 (\kappa^\star)^2}$ for a small enough universal constant $c$, any local minimum $U$ of Objective~\eqref{eqn:symmetricobj} satisfies $\|UU^\top-\Ms\|_F^2 \le \epsilon \|\Ms\|_F^2$. 
\end{lemma}

\begin{proof}
By Lemma~\ref{lem:gjzmain_sym}, we know that every local minimum of $f(U)$ satisfies
\[
\|\Delta\Delta^\top\|_W^2 - 3\|UU^\top - \Ms\|_W^2 + \left([\nabla^2 Q(U)](\Delta) - 4\inner{\nabla Q(U), \Delta}\right) \ge 0.
\]

We will bound these three terms. First, by Lemma~\ref{lem:symmetricnormbound}, we know $\normtwo{\Delta_i}^2 \le O\left(\frac{\mu^2 r^2 \kappa^\star \sigs_1}{n}\right)$. 

For the first term $\|\Delta\Delta^\top\|_W^2$, we can directly apply Lemma~\ref{lem:deterministc_main}:
\begin{align*}
\|\Delta\Delta^\top\|_W^2 & \le \|\Delta\Delta^\top\|_F^2 + \|W-J\|\|\Delta\|_F^2 \max_i \|\Delta_i\|^2 \\
& \le \|\Delta\Delta^\top\|_F^2 + \|W-J\|\cdot O\left(\frac{\mu^2 r^2 \kappa^\star \sigs_1}{n}\right)\cdot \|\Delta\|_F^2 \\
& \le \|\Delta\Delta^\top\|_F^2 + 0.1\sigs_r \|\Delta\|_F^2.
\end{align*}

The last inequality uses the fact that $\|W-J\| \le \frac{cn}{\mu^2 r^2 (\kappa^\star)^2}$ for a small enough constant $c$. 

For the second term $\|UU^\top - \Ms\|_W^2$, we invoke Lemma~\ref{lem:deterministc_main} with $X = (U, \Us)$ and $Y = (U, -\Us)$.
Notice that $X Y^\top = UU^\top - \Ms$.
Moreover, we know that $\fnorm{X} \le \|U\|_F + \|\Us\|_F \le 2\|\Us\|_F + \|\Delta\|_F$. Similarly, the row norms of $X$ is still upper bounded by $O\left(\frac{\mu^2 r^2 \kappa^\star \sigs_1}{n}\right)$. Therefore,
\begin{align*}
\|UU^\top - \Ms\|_W^2 & \ge \|UU^\top - \Ms\|_F^2 - \|W-J\|\|(U,\Us)\|_F^2 \max_i \normtwo{(U,\Us)_i}^2 \\
& \ge \|UU^\top - \Ms\|_F^2 - \|W-J\|\cdot O\left(\frac{\mu^2 r^2 \kappa^\star \sigs_1}{n}\right)\cdot (2\|\Us\|_F + \|\Delta\|_F)^2 \\
& \ge \|UU^\top - \Ms\|_F^2 - 0.1 \epsilon\sigs_r \|\Us\|_F^2 -0.1\sigs_r \|\Delta\|_F^2.
\end{align*}

Again, the last step uses the fact that $\|W-J\| \le \frac{\eps cn}{\mu^2 r^2 (\kappa^\star)^2}$ for a small enough constant $c$. 

Finally, the third term is bounded by $0.1\sigs_r\|\Delta\|_F^2$ by Lemma~\ref{lem:extra_bound_symmetric}. Combining all these terms,
\begin{align*}
0 & \le \|\Delta\Delta^\top\|_W^2 - 3\|UU^\top - \Ms\|_W^2 + \left([\nabla^2 Q(U)](\Delta) - 4\inner{\nabla Q(U), \Delta}\right) \\
& \le \left(\|\Delta\Delta^\top\|_F^2 + 0.1\sigs_r \|\Delta\|_F^2\right) - 3\left(\|UU^\top - \Ms\|_F^2 - 0.1 \epsilon\sigs_r \|\Us\|_F^2 -0.1\sigs_r \|\Delta\|_F^2\right) + 0.1\sigs_r \|\Delta\|_F^2 \\
& \le -\|UU^\top - \Ms\|_F^2 + 0.5 \sigs_r\|\Delta\|_F^2 + 0.3\epsilon \sigs_r \|\Us\|_F^2 \\
& \le -0.3\|UU^\top - \Ms\|_F^2 + 0.3 \epsilon\sigs_r \|\Us\|_F^2.
\end{align*}

Here the calculations use the inequalities in Lemma~\ref{lem:normconnect}. 

As a result, $\|UU^\top - \Ms\|_F^2 \le \epsilon \sigs_r \|\Us\|_F^2$.
We conclude the proof by noting that $\sigs_r \|\Us\|_F^2 = \sigs_r \sum_{i=1}^r \sigs_i \le \sum_{i=1}^r (\sigs_i)^2 = \|\Ms\|_F^2$, because the singular values of $\Us$ are $\sqrt{\sigs_i}$'s.
\end{proof}

\subsection{Our Asymmetric Case: Proof of Theorem~\ref{thm:asymmetric_local}}
\label{app:matrix-asymmetric}

The proof for the asymmetric case (Objective~\eqref{eqn:asymmetricobj}) is very similar.
Recall that $\Ms = \Us(\Vs)^\top$, where $\Us \in \R^{n_1\times r}$ and $\Vs = \R^{n_2\times r}$.
We again start by bounding the row norms of $U$ and $V$.

\begin{lemma} \label{lem:asymmetricnormbound}
Suppose $\norminf{W} \le n_2$ and $\normone{W} \le n_1$. Choose $\alpha_1^2 = \frac{C\mu r\sigs_1}{n_1},\alpha_2^2 = \frac{C\mu r\sigs_1}{n_2}$ and $\lambda_1 = \frac{C^2 n_1}{\mu r\kappa^\star},\lambda_2 = \frac{C^2 n_1}{\mu r\kappa^\star}$ where $C$ is a large enough universal constant. For $f$ as in Objective~\eqref{eqn:asymmetricobj} and any matrix $Z = \begin{pmatrix}
U \\ V
\end{pmatrix}$ with $\grad f(Z) = 0$, we have
\begin{equation*}
\max_{i}\normtwo{U_i}^2 = O\left(\frac{\mu^3 r^3 (\kappa^\star)^2 \sigs_1}{n_1}\right); \quad 
\max_{i}\normtwo{V_i}^2 = O\left(\frac{\mu^3 r^3 (\kappa^\star)^2 \sigs_1}{n_2}\right).
\end{equation*}
\end{lemma}

\begin{proof}
Without loss of generality, we assume $\sqrt{n_1}\max_i\normtwo{U_i} \ge \sqrt{n_2}\max_i\normtwo{V_i}$, so it is enough to upper bound $\max_i \normtwo{U_i}$. The gradient can be computed as follows.
\begin{align*}
\grad f(Z) & =
4 \begin{pmatrix}
[W*(M -\Ms)] V \\
[W*(M -\Ms)]^\top U
\end{pmatrix}
+
2 \begin{pmatrix}
 U(U^\top U - V^\top V) \\
 V(V^\top V - U^\top U)
\end{pmatrix}
+
 \grad Q(Z),
\end{align*}
where 
\begin{equation*}
\grad Q(Z) = 4\lambda_1 \sum_{i=1}^{n_1}(\normtwo{Z_i} - \alpha_1)^3_{+}\frac{e_i Z_i}{\normtwo{Z_i}^2} 
+ 4\lambda_2 \sum_{i=n_1+1}^{n_1+n_2}(\normtwo{Z_i} - \alpha_2)^3_{+}\frac{e_i Z_i}{\normtwo{Z_i}^2}.
\end{equation*}

First, we observe that $\inner{\nabla Q(Z), Z} \ge 0$, therefore,
\begin{align*}
0 & = \inner{\nabla f(Z), Z} \\
& = 2 \|U^\top U - V^\top V\|_F^2 + 8 \inner{M - \Ms, M}_W + \inner{\nabla Q(Z),Z} \\
& \ge 2 \|U^\top U - V^\top V\|_F^2 - 8 \inner{\Ms, M}_W \\
& \ge 2 \|U^\top U - V^\top V\|_F^2 - 8n_1n_2 \maxnorm{\Ms} \maxnorm{M}. \\
& \ge 2 \|U^\top U - V^\top V\|_F^2 - 8\sqrt{n_1n_2}\mu r\sigs_1 \maxnorm{M}.
\end{align*}

Let $i^\star = \arg\max_i \normtwo{U_i}$, if $\normtwo{U_{i^\star}} \le 2 \alpha_1$ then we are done. On the other hand, if $\normtwo{U_{i^\star}} > 2\alpha_1$, we know $\maxnorm{M} \le \max_{i,j} \normtwo{U_i}\normtwo{V_j} \le \sqrt{n_1/n_2} \normtwo{U_{i^\star}}^2$. 
As a result, we know 
\[
\|U^\top U - V^\top V\|_F^2 \le 4\sqrt{n_1n_2}\mu r\sigs_1\|M\|_\infty \le 4n_1\mu r \sigs_1\normtwo{U_{i^\star}}^2.
\]

Let $Q(Z) = Q(U,V) = Q_1(U)+Q_2(V)$.
Consider the gradient of $f(Z)$ along the direction $e_{i^\star} Z_{i^\star}$, we have
\begin{align*}
0 & = \inner{\nabla f(Z), e_{i^\star} Z_{i^\star}} = \inner{e_{i^\star}^\top\nabla f(Z), Z_{i^\star}} \\
& = \inner{e_{i^\star}^\top\left[4(W*(M-\Ms))V + 2 U(U^\top U-V^\top V) + \nabla Q_1(U)\right], U_{i^\star}} \\
& \ge 4\lambda_1(\normtwo{U_{i^\star}}-\alpha_1)_+^3\normtwo{U_{i^\star}} - 4\inner{e_{i^\star}^\top \Ms, e_{i^\star}^\top M}_W - 2 \|U^\top U-V^\top V\|_F \normtwo{U_{i^\star}}^2 \\
& \ge \frac{\lambda_1}{2}\normtwo{U_{i^\star}}^4 - 4n_2 \maxnorm{\Ms} \maxnorm{M} - 2 \|U^\top U-V^\top V\|_F \normtwo{U_{i^\star}}^2. \\
& \ge \frac{\lambda_1}{2}\normtwo{U_{i^\star}}^4 - 4\mu r\sigs_1\normtwo{U_{i^\star}}^2 - 4\sqrt{n_1\mu r \sigs_1}\normtwo{U_{i^\star}}^3. 
\end{align*}

This implies 
\[
\frac{\lambda_1}{2}\normtwo{U_{i^\star}}^2 \le 4\mu r\sigs_1 + 4\sqrt{n_1\mu r \sigs_1}\normtwo{U_{i^\star}}.
\]
Therefore, $\normtwo{U_{i^\star}}^2 = O\bigl(\max\{\frac{\mu r\sigs_1}{\lambda_1}, \frac{n_1\mu r \sigs_1}{\lambda_1^2}\}\bigr) = O\left(\frac{\mu^3 r^3 (\kappa^\star)^2 \sigs_1}{n_1}\right)$ by our choices of $\alpha$'s and $\lambda$'s. 
\end{proof}

Next, we will prove that all local minima are close to the ground truth.

\begin{lemma}
Fix any error parameter $0 < \eps < 1$.
Suppose the weight matrix $W$ satisfies that $\norminf{W} \le n_2$, $\normone{W} \le n_1$, and $\|W-J\| \le \frac{\eps c\sqrt{n_1n_2}}{\mu^3 r^3 (\kappa^\star)^3}$ for a small enough universal constant $c$. Then, any local minimum $(U, V)$ of Objective~\eqref{eqn:asymmetricobj} has $\|UV^\top-\Ms\|_F^2 \le \epsilon \|\Ms\|_F^2$. 
\end{lemma}

\begin{proof}
By Lemma~\ref{lem:gjzmain_asym} we know for every local minimum of $f(U,V)$ satisfies
\[
\|\Delta\Delta^\top\|_{\bar{W}}^2 - 3\|N - N^\star\|_{\bar{W}}^2 + \left([\nabla^2 Q(Z)](\Delta) - 4\inner{\nabla Q(Z), \Delta}\right) \ge 0.
\]

We will bound these three terms. First, by Lemma~\ref{lem:asymmetricnormbound} we know the rows of $\Delta_U$ have squared $\ell_2$-norm at most $ O\left(\frac{\mu^3 r^3 (\kappa^\star)^2 \sigs_1}{n_1}\right)$, and the rows of $\Delta_V$ have squared $\ell_2$-norm at most $ O\left(\frac{\mu^3 r^3 (\kappa^\star)^2 \sigs_1}{n_2}\right)$.

For the first term $\|\Delta\Delta^\top\|_{\bar{W}}^2$, by the definition of $\bar W$,
\begin{align*}
\norm{\Delta \Delta^\top}_{\bar W}^2 & = \norm{\Delta_U \Delta_U^\top}_F^2 + \norm{\Delta_V \Delta_V^\top}_F^2 - 2 \norm{\Delta_U \Delta_V^\top}_F^2 + 4 \norm{\Delta_U \Delta_V^\top}_W^2 \\
  & = \norm{\Delta \Delta^\top}_{F}^2 + 4(\norm{\Delta_U \Delta_V^\top}_W^2 - \norm{\Delta_U \Delta_V^\top}_F^2).
\end{align*}
We can directly apply Lemma~\ref{lem:deterministc_main} to $\norm{\Delta_U \Delta_V^\top}_W^2$.
\begin{align*}
\|\Delta_U\Delta_V^\top\|_W^2 & \le \|\Delta_U\Delta_V^\top\|_F^2 + \|W-J\| \cdot \|\Delta_U\|_F \cdot \|\Delta_V\|_F \cdot \max_{i=1}^{n_1} \normtwo{\Delta_i} \cdot \max_{j=n_1+1}^{n_1+n_2} \normtwo{\Delta_j} \\
& \le \|\Delta_U\Delta_V^\top\|_F^2 + \|W-J\|\cdot O\left(\frac{\mu^3 r^3 (\kappa^\star)^2 \sigs_1}{\sqrt{n_1n_2}}\right)\cdot \|\Delta\|_F^2 \\
& \le \|\Delta_U\Delta_V^\top\|_F^2 + 0.01\sigs_r \|\Delta\|_F^2.
\end{align*}

Here the last inequality uses the fact that $\|W-J\| \le \frac{c\sqrt{n_1n_2}}{\mu^3 r^3 (\kappa^\star)^3}$ for a small enough constant $c$. 

For the second term, we can relate $\bar W$-norm to $W$-norm similarly, which allows us to focus on the $W$-norm of the off-diagonal blocks, $\|UV^\top - \Ms\|_W^2$.
We then invoke Lemma~\ref{lem:deterministc_main} with $X = (U, \Us)$ and $Y = (V, -\Us)$.
We know $X Y^\top = U V^\top - \Ms$ and $\|X\|_F \le \|U\|_F + \|\Us\|_F \le 2\|\Us\|_F + \|\Delta\|_F$ (the same upper bound holds for $\|Y\|_F$ because $\|\Us\|_F = \|\Vs\|_F$). The row norms of $X$ is still bounded by $O\left(\frac{\mu^2 r^2 \kappa^\star \sigs_1}{n_1}\right)$ (and similarly for $Y$ except the denominator is $n_2$).
\begin{align*}
& \|UV^\top - \Ms\|_W^2 \\
& \ge \|UV^\top - \Ms\|_F^2 - \|W-J\| \cdot \|(U,\Us)\|_F \|(V,-\Vs)\|_F \cdot \max_i \normtwo{(U,\Us)_i} \cdot \max_j \normtwo{(V,\Vs)_j} \\
& \ge \|UV^\top - \Ms\|_F^2 - \|W-J\| \cdot O\left(\frac{\mu^3 r^3 (\kappa^\star)^2 \sigs_1}{\sqrt{n_1n_2}}\right)\cdot (2\|\Us\|_F + \|\Delta\|_F)^2 \\
& \ge \|UV^\top - \Ms\|_F^2 - 0.05 \epsilon\sigs_r \|\Us\|_F^2 -0.01\sigs_r \|\Delta\|_F^2.
\end{align*}

Again, the last step uses the fact that $\|W-J\| \le \frac{\eps cn}{\mu^3 r^3 (\kappa^\star)^3}$ for a small enough constant $c$. 

Finally, the third term is bounded by $0.1\sigs_r\|\Delta\|_F^2$ by Lemma~\ref{lem:extra_bound_asymmetric}. We combine all these terms and apply Lemma~\ref{lem:normconnect},
\begin{align*}
0 &\le \|\Delta\Delta^\top\|_{\bar{W}}^2 - 3\|N- N^\star\|_{\bar{W}}^2 + \left([\nabla^2 Q(Z)](\Delta) - 4\inner{\nabla Q(Z), \Delta}\right) \\
& \le \left(\|\Delta\Delta^\top\|_F^2 + 0.04\sigs_r \|\Delta\|_F^2\right) - 3\left(\|N- N^\star\|_F^2 - 0.2 \epsilon\sigs_r \|\Us\|_F^2 -0.01\sigs_r \|\Delta\|_F^2\right) + 0.1\sigs_r \|\Delta\|_F^2\\
& \le -\|N- N^\star\|_F^2 + 0.17 \sigs_r\|\Delta\|_F^2 + 0.6 \epsilon \sigs_r \|\Us\|_F^2 \\
& \le -0.6\|N- N^\star\|_F^2 + 0.6 \epsilon\sigs_r \|\Us\|_F^2.
\end{align*}

As a result, $\|M-\Ms\|_F^2 \le \|N-N^\star\|_F^2 \le \epsilon \sigs_r \|\Us\|_F^2 \le \epsilon \|\Ms\|_F^2$. 
\end{proof}

\end{document}